\documentclass[a4paper]{article}
\usepackage{geometry}
\usepackage{times}
\usepackage{epsfig}
\usepackage{graphicx}
\usepackage{amsmath}
\usepackage{amssymb}
\usepackage{amsthm}
\usepackage[tight,footnotesize]{subfigure}
\usepackage{algorithm}
\usepackage{algorithmic}
\usepackage{array}
\usepackage{multirow}
\usepackage[noblocks]{authblk}
\usepackage{url}

\newcommand{\mb}[1]{\mathbf{#1}}
\newcommand{\tb}[1]{\textbf{#1}}

\DeclareMathOperator{\rank}{rank}
\DeclareMathOperator{\diag}{diag}

\newcommand{\PreserveBackslash}[1]{\let\temp=\\#1\let\\=\temp}
\newcolumntype{C}[1]{>{\PreserveBackslash\centering}p{#1}}
\newcolumntype{R}[1]{>{\PreserveBackslash\raggedleft}p{#1}}
\newcolumntype{L}[1]{>{\PreserveBackslash\raggedright}p{#1}}

\newtheorem{theorem}{Theorem}
\newtheorem{lemma}[theorem]{Lemma}
\newtheorem{proposition}[theorem]{Proposition}
\newtheorem{corollary}[theorem]{Corollary}
\newtheorem{definition}{Definition}

\begin{document}

\title{Spectral-graph Based Classifications: Linear Regression for Classification and Normalized Radial Basis Function Network}

%

\author[*]{Zhenfang~Hu}
\author[*]{Gang~Pan}
\author[*]{Zhaohui~Wu}

\affil[*]{College of Computer Science and Technology, Zhejiang
University, China, \authorcr fancij@zju.edu.cn, gpan@zju.edu.cn, wzh@zju.edu.cn}

\maketitle

\begin{abstract}
Spectral graph theory has been widely applied in unsupervised and semi-supervised learning. It is still unknown how it can be exploited in supervised learning. In this paper, we find for the first time, to our knowledge, that it also plays a concrete role in supervised classification. It turns out that two classifiers are inherently related to the theory: linear regression for classification (LRC) and normalized radial basis function network (nRBFN), corresponding to linear and nonlinear kernel respectively. The spectral graph theory provides us with a new insight into a fundamental aspect of classification: the tradeoff between fitting error and overfitting risk. With the theory, ideal working conditions for LRC and nRBFN are presented, which ensure not only zero fitting error but also low overfitting risk. For quantitative analysis, two concepts, the fitting error and the spectral risk (indicating overfitting), have been defined. Their bounds for nRBFN and LRC are derived. A special result shows that the spectral risk of nRBFN is lower bounded by the number of classes and upper bounded by the size of radial basis. When the conditions are not met exactly, the classifiers will pursue the minimum fitting error, running into the risk of overfitting. It turns out that $\ell_2$-norm regularization can be applied to control overfitting. Its effect is explored under the spectral context. It is found that the two terms in the $\ell_2$-regularized objective are one-one correspondent to the fitting error and the spectral risk, revealing a tradeoff between the two quantities. Concerning practical performance, we devise a basis selection strategy to address the main problem hindering the applications of (n)RBFN. With the strategy, nRBFN is easy to implement yet flexible. Experiments on 14 benchmark data sets show the performance of nRBFN is comparable to that of SVM, whereas the parameter tuning of nRBFN is much easier, leading to reduction of model selection time.
\end{abstract}

\tb{Keywords:}
  classification, spectral graph, radial basis function network, linear regression, regularization, overfitting.

\section{Introduction}

Spectral graph theory is a theory that centers around the
graph Laplacian matrix \cite{chung1997spectral}. On the one hand,
it can reveal underlying cluster structure of data by the eigenvectors of the Laplacian matrix, on the other
hand, the eigenvectors can serve as dimensionally reduced codes that preserve pair-wise data relation. The theory has found wide applications in unsupervised learning,
including clustering \cite{von2007tutorial} (generally named spectral clustering, including, e.g., ratio cut (Rcut) \cite{chan1994spectral} and normalized cut (Ncut) \cite{shi2000normalized,ng2002spectral}), and dimensionality reduction (e.g.,
Laplacian eigenmap (LE) \cite{belkin2003laplacian} and locality preserving projections (LPP) \cite{he2003locality}). Later, it develops as a popular paradigm in semi-supervised learning, including semi-supervised clustering
\cite{Kamvar2003Spectral,Kulis2005Semi} and semi-supervised classification
\cite{Zhu2003Semi,Zhou2004Learning,Belkin2006Manifold,Zhu2008Semi}. In semi-supervised learning, in an attempt to impose pair-wise data relation, the role of spectral graph usually appears as a ``graph-regularization'' term added to the other objectives.

Recently it has been discovered that, in the scope of unsupervised learning, spectral graph theory unifies a series of elementary methods of machine learning into a complete framework \cite{Hu2014Spectral}. The methods cover dimensionality reduction, cluster analysis, and sparse representation. They range from principal component analysis (PCA) \cite{jolliffe2002principal}, K-means \cite{macqueen1967some}, LE \cite{belkin2003laplacian}, Rcut \cite{chan1994spectral}, and a new spectral sparse representation (SSR) \cite{Hu2014Spectral}. It is revealed that these methods share inherent relations, they even become equivalent under an ideal graph condition. The framework also incorporates extended relations to conventional over-complete sparse representations \cite{Elad2010Sparse}, e.g., \cite{olshausen1996emergence}, method of optimal directions (MOD) \cite{Engan1999Method}, KSVD \cite{aharon2006img}; manifold learning, e.g., kernel PCA \cite{scholkopf1998nonlinear}, multidimensional scaling (MDS) \cite{cox2001multidimensional}, Isomap \cite{tenenbaum2000global}, locally linear embedding (LLE) \cite{roweis2000nonlinear}; and subspace clustering, e.g., sparse subspace clustering (SSC) \cite{elhamifar2013sparse}, low-rank representation (LRR) \cite{liu2013robust}.

However, as far as we know, spectral graph theory has not yet found applications in supervised classification, except in hybrid-model way where the relation is not inherent, e.g., the addition of certain classification objective with a graph-regularization term. It is interesting to know whether the theory plays a concrete role in supervised classification and which classifiers are related.

In supervised learning, linear regression and radial basis function network (RBFN) are two basic methods. Both of them are devoted to function fitting, including classification as special case. It is well-known that linear regression can be interpreted with Bayesian probability (see, e.g., \cite{bishop2006pattern}), and when used for classification (LRC), its link to linear discriminant analysis was already discovered \cite{Ye2007Least}. RBFN \cite{Powell1987Radial,Broomhead1988Multivariable,Moody1989Fast} and its normalized variant (nRBFN) \cite{Moody1989Fast,Specht1991A} are classical neural networks well-known for their simple structures and universal function approximation capacities \cite{Park1991Universal,Xu1994On,Benaim1994On}. Among broad connections to many theories \cite{Blanzieri2003Theoretical}, RBFN can be interpreted with Tikhonov regularization theory \cite{Poggio1990Networks}, while nRBFN can be interpreted with kernel regression \cite{Specht1991A,Xu1994On}. However, the above methods are hardly related to spectral graph theory before.

In classification application, it has been repeatedly demonstrated that \cite{Rifkin2004In,Oyang2005Data,Fern2014Do,Que2016Back} the performance of RBFN can be comparable to support vector machine (SVM) \cite{Cortes1995Support}. Especially, in a comprehensive evaluation involving 179 classifiers and 121 data sets by \cite{Fern2014Do}, RBFN ranks third, immediately following SVM.\footnote{In that evaluation, the classifier ranking third is a kernel version of extreme learning machine (ELM) \cite{Huang2006Extreme}, named ``elm\_kernel\_m'' in the evaluation, but we carefully checked the codes of ``elm\_kernel\_m'' and found that it is actually a classical RBFN that uses the whole training set as the basis. The only difference is that the target indicators are converted to be of values 1 (true class) and -1 (other classes). This RBFN is a special case of ELM.} However, despite of bearing additional advantages, (n)RBFN did not receive wide applications as SVM. This situation may be attributed to: (1) the less sound theoretical background compared with SVM, (2) the basis selection problem involved in (n)RBFN design \cite{Chen1991Orthogonal,Orr1995Regularization,Scholkopf1997Comparing,Oyang2005Data,Que2016Back}, and (3) the difficulty of parameter tuning. In this paper, in addition to exploring stronger theoretical background for (n)RBFN, to make (n)RBFN become practical tool, we propose a solution scheme for the basis selection and parameter setting problems.

\subsection{Our Work}
In this paper, we uncover the concrete role of spectral graph theory in supervised classification, and find that LRC and nRBFN are inherently related to the theory. The tradeoff between fitting error and overfitting risk is a fundamental problem of classification. The theory provides us with a new insight into this problem under the context of LRC and nRBFN. With the theory, we establish the ideal working conditions for the two classifiers, which ensure not only zero fitting error but also low overfitting risk. When the conditions are not met exactly, the $\ell_2$-norm regularization can be applied to control overfitting, its effect is revealed under the spectral context. As a benefit, the regularization weight can be set in a principled and easy way. The followings are more detailed introduction.

In spectral clustering, we directly extract the cluster information from the eigenvectors of Laplacian matrix, i.e., recover the indicator vectors from the eigenvectors. However, in classification, the partition of data is assigned, and the indicator vectors are given, it seems not straightforward to see how spectral graph theory will work in this case. It turns out that we are to find the closest components in the eigenspace of Laplacian matrix to approximate the given indicator vectors. If the data is well-behaved, i.e., the given classes match the underlying clusters, then unsupervised clustering and supervised classification become consistent, and they unify under the spectral graph framework.

In this paper, we find that LRC and nRBFN are inherently related to spectral graph theory. Firstly, LRC will lay down some theoretical foundation, then nRBFN is derived via applying kernel trick on LRC. Broadly speaking, the data/feature matrix used by LRC/nRBFN shares the same eigenvectors with the Laplacian matrix, and we are to find the closest components in the eigenvectors to approximate the given indicator vectors. When an ideal graph condition is satisfied, which requires the classes being totally separated, the indicator vectors appear in the leading eigenspace of largest eigenvalues, and consequently zero fitting error is achieved. That is the inherent relation of LRC and nRBFN to spectral graph theory.

Although, zero fitting error is desirable, classification is more concerned with generalization performance. Striking a balance between low fitting error and low overfitting risk is a critical problem. Under the ideal graph condition, things are perfect, the fitting error is zero and the overfitting risk is low. From qualitative point of view, this is because the indicator vectors are found in the principal subspace. The principal subspace corresponds to stable features, as contrary to the minor subspace that corresponds to noisy features, especially when sampling is insufficient. For quantitative analysis, we define two concepts: the fitting error and the spectral risk. The spectral risk measures the deviation of the found components to the principal subspace, therefore it signals a warning of overfitting. The bounds of the two quantities for nRBFN and LRC are derived. A special result shows that the spectral risk of nRBFN is lower bounded by the number of classes--a quantity representing ``problem complexity'', and upper bounded by the size of basis--a quantity representing ``model complexity''. The upper bound indicates a tradeoff between fitting error and overfitting risk: larger basis implies lower error but higher risk.

In practice, the ideal condition cannot be met exactly, the leading eigenspace will deviate from the target indicator vectors. The found closest components may lie in minor subspace of small singular values. It is easily prone to noise, giving rise to increment of overfitting risk. It will be shown that the $\ell_2$-norm regularization can alleviate this problem. Its effect is explored under the spectral context. First, qualitatively, it drives the classifier to find the closest components in the principal subspace and discourages the opposite direction. Second, quantitatively, the two terms in the $\ell_2$-regularized objective are in one-one correspondence to the fitting error and the spectral risk, showing a tradeoff between the two quantities.

nRBFN is more powerful than LRC for its nonlinear kernel and significant risk bounds, we thus focus on nRBFN. To make nRBFN work in practice, we devise a basis selection strategy to address the main problem that hinders the wide applications of (n)RBFN. The strategy is based on soft K-nearest neighbors. It is easy to implement and the result is deterministic, in contrast to traditional K-means based strategy that depends on random initialization. Traditionally, setting the basis size is a troublesome problem. In our scheme, it is implicitly determined via a user-friendly threshold within range $(0,1]$. With this threshold, the basis size can be automatically determined according to the complexity of data distribution. We can also flexibly control the tradeoff between accuracy and computational cost via this threshold. Besides, when using $\ell_2$-norm regularization, the regularization weight can be set in a more principled way. In all, the parameter tuning is easy, leading to significant reduction of model selection time.

The contributions of the paper are as follows:
\begin{enumerate}
  \item We have extended the spectral graph framework to the supervised classification domain. We found for the first time, to our knowledge, that spectral graph theory plays a concrete role in supervised classification. Two classifiers, LRC and nRBFN, corresponding to linear and nonlinear kernel respectively, turn out to be inherently related to the theory. With the theory, ideal working conditions for LRC and nRBFN are presented, which ensure not only zero fitting error but also low overfitting risk.

  \item With the spectral graph theory, new insights into the overfitting problem as well as the effect of $\ell_2$-norm regularization are obtained. For quantitative analysis, two concepts, the fitting error and the spectral risk have been defined. The bounds of them for nRBFN and LRC have been derived. One result states that the spectral risk of nRBFN is lower bounded by the number of classes and upper bounded by the size of radial basis. In addition, it turns out that the two terms in the $\ell_2$-regularized objective are one-one correspondent to the fitting error and the spectral risk, revealing a tradeoff between the two quantities.

  \item We have devised a basis selection strategy for (n)RBFN, so that nRBFN becomes easy to implement yet flexible. The performance of nRBFN is comparable to that of SVM, whereas the parameters of nRBFN are much easier to set, leading to significant reduction of model selection time.
\end{enumerate}

The rest of the paper is organized as follows. Section~\ref{sec:clustering} briefly introduces spectral graph theory and reviews its rationale for clustering. Section~\ref{sec:classification} presents spectral-graph based classifications, including linear version LRC and kernel version nRBFN. Meanwhile, the ideal working conditions are introduced. In the cases of the conditions are not met exactly, Section~\ref{sec:regularization} introduces the $\ell_2$-norm regularization for LRC and nRBFN. Section~\ref{sec:risk} defines the fitting error and spectral risk, derives their bounds for nRBFN and LRC, and reveals the effect of $\ell_2$-norm regularization. Section~\ref{sec:basis selection} proposes the basis selection strategy. Section~\ref{sec:experiment} demonstrates the performance of nRBFN and empirically evaluates the fitting error and spectral risk. Section~\ref{sec:related work} introduces some related work. The paper is ended with further work in Section~\ref{sec:future work}.

\tb{Notations}. $A=[A_1,\dots,A_n]\in \mathbb{R}^{p\times n}$: data matrix with $n$
samples of dimension $p$. $F=[F_1,\dots,F_n]\in \mathbb{R}^{K\times n}$: indicator matrix
for $n$ samples of $K$ classes. If the $i$th sample belongs to class
$k$, then the $k$th entry of $F_i$ is one and the others are zero.
$G=[G_1,\dots,G_r]\in \mathbb{R}^{p\times r}$: basis vectors of
RBFN and nRBFN. $\mb{1}$: a vector of uniform value 1. $\diag(v)$: a diagonal matrix formed by vector
$v$.

\section{Spectral-graph Based Clustering: a
Review}\label{sec:clustering}
Given an undirected graph of $n$
vertices (data points), with the adjacency matrix defined to be a
similarity matrix $W\in \mathbb{R}^{n\times n}$, measuring the pairwise similarities between data points, $W_{ij}=W_{ji}\geq 0$, the Laplacian matrix is defined as $L\doteq S-W$, where $S$
is a diagonal degree matrix with the diagonal being the sum of
weights of each vertex, i.e., $S=\diag(\mb{1}^TW)$. The Laplacian
matrix has the following properties \cite{von2007tutorial}.\footnote{These properties are
not shared by the similarity matrix.}
\begin{enumerate}
  \item It is positive semi-definite.
  \item Vector $\mb{1}$ is always an eigenvector with eigenvalue zero.
  \item Assume there are $K$ connected components in the graph, then the indicator vectors of these components (row
vectors of $F$) span the eigenspace of
eigenvalue zero.
\end{enumerate}

These properties are exploited for clustering purpose \cite{von2007tutorial}. Assume we are to find
$K$ clusters, if the ideal graph condition for clustering (Definition~\ref{def:ideal clustering}) \cite{Hu2014Spectral} holds (the condition implies the between-cluster weights are all zero: $W_{ij}=0$, if the $i$th and $j$th points are of different clusters), then we can compute the
$K$ eigenvectors of $L$ with the smallest eigenvalues (zero), and
then postprocess these eigenvectors to finish
clustering.
In practice,
the $K$ components of the graph may not be completely disconnected. In this noisy case,
the same procedure can still be applied, since the $K$
eigenvectors with the smallest eigenvalues become rotated noisy
indicators, which may not differ much to their ideal ones. This is the working rationale of spectral clustering.

\begin{definition}\label{def:ideal clustering}
\tb{(ideal graph condition for clustering)} Targeting for $K$
clusters, if there are exactly $K$ connected components in the
graph, then the graph (or similarity matrix) is called ideal (with
respect to $K$ clusters).
\end{definition}

Finally, the eigenvectors of $L$ with eigenvalue
zero (smallest) are the eigenvectors of $S^{-1}W$, called normalized
Laplacian matrix, with eigenvalue one (largest) \cite{von2007tutorial}.

\section{Spectral-graph Based Classifications}\label{sec:classification}
In spectral clustering, we directly extract the cluster information from the eigenvectors of Laplacian matrix, i.e., recover the indicator vectors from the eigenvectors. However, in classification, the indicator vectors are given. It will be shown that we are to find the closest components in the eigenspace to approximate the indicator vectors. When an ideal condition is satisfied, the indicator vectors appear in the leading eigenspace, achieving zero fitting error and low overfitting risk.
\subsection{Linear Version:
Linear Regression for Classification (LRC)}\label{sec:lrc}
We will show that the singular vectors of the data matrix are the eigenvectors of a Laplacian matrix. Thus the link to spectral graph theory is established. The Laplacian matrix is built by the inner product between data, i.e., linear kernel. In the following, we first introduce the basic formulation of LRC, then analyze it from the row-space view, this leads to the relation to spectral graph theory. Based on the theory, an ideal working condition for LRC is presented. Finally, we analyze LRC from the column-space view, which paves the way to nRBFN.
\subsubsection{Basic Formulation}
Given data matrix $A$ (assume mean-removed, $A\mb{1}=0$) and the
corresponding class labels, we convert the labels to an
indicator matrix $F$, and define an augmented data matrix
$\tilde{A}=\begin{bmatrix}\sqrt{\beta}\mb{1}^T\\A\end{bmatrix}$ \cite{Hu2014Spectral},
where $\beta$ is a constant scalar that will be introduced later.
The objective of LRC is to find a weight matrix $D$ so that the linear
combinations of the columns of $D$ and the samples approximate the
indicator vectors:\footnote{It is equivalent to the classical LRC
where $\beta=1$ \cite{bishop2006pattern}, since $\sqrt{\beta}$ can be absorbed into the first
column of $D$. In implementation, we indeed use $\beta=1$.}
\begin{equation}\label{equ:lrc}
\min_{D}\, \sum_{i=1}^n
\|F_i-D\tilde{A}_i\|_2^2=\|F-D\tilde{A}\|_F^2.
\end{equation}
Provided $\rank(\tilde{A})=p+1$, there is a unique closed-form
solution: $D^*=F\tilde{A}^T(\tilde{A}\tilde{A}^T)^{-1}$ \cite{bishop2006pattern}. After $D^*$
is obtained, given a test sample $b$ (with mean removed as $A$), its label is determined by the
maximum entry of $D^*\tilde{b}$:
\begin{equation}\label{equ:lrc b}
\mathop{\arg\max}_k\,
(D^*\tilde{b})_k.
\end{equation}
It can be shown that the sum of $D^*\tilde{b}$ is always one.\footnote{Referring to the proof of Lemma~\ref{lem:1X} in Section~\ref{sec:row-col srbf}, it can be shown that $\mb{1}^TD^*=[1,0,\dots,0]$.}
Besides, $D^*\tilde{b}$ is an approximation to the indicator vector.
These endow $D^*\tilde{b}$ with a quasi-probability interpretation
(may include negative values).

\subsubsection{Row-space View}
It will be shown that LRC are to find the closest components in the data row-space to approximate the indicator vectors, and this relates to the spectral graph.

Substituting $D^*$ into (\ref{equ:lrc}), we obtain
$\|F-F\tilde{A}^T(\tilde{A}\tilde{A}^T)^{-1}\tilde{A}\|_F^2$. Note
that $\tilde{A}^T(\tilde{A}\tilde{A}^T)^{-1}\tilde{A}$ is a
projection matrix, and the projection subspace is spanned
by the rows of $\tilde{A}$. In this view, to approximate $F$, LRC
projects $F$ onto the row-space of $\tilde{A}$ and reconstructs. These facts are well-known. However, a
natural question arises: does the row-space contain ``ingredients''
close to $F$, so that the reconstruction error is small? With spectral graph theory, we present an ideal condition under which the row-space of $\tilde{A}$ contains $F$. Define a similarity matrix $W\doteq \tilde{A}^T\tilde{A}$, and
\begin{equation}\label{equ:beta}
\beta\doteq -\min_{ij}\,(A^TA)_{ij}.
\end{equation}
Since the data is mean-removed, we have $\min_{ij}\,(A^TA)_{ij}<0$,\footnote{Otherwise $A^TA\mb{1}\neq \mb{0}$ violating $A\mb{1}=\mb{0}$.} $\beta$ thus defined makes $W$ become nonnegative \cite{Hu2014Spectral}. The condition and theorem are as follows:\footnote{Note that the theory requires the data to be mean-removed, which may be ignored by traditional LRC.}

\begin{definition}\label{def:ideal classification}
\tb{(ideal graph condition for classification)} If the weights
between vertices of different classes are all zero, i.e., $\forall
i,j$, $W_{ij}=0$ if $F_i\neq F_j$, then the graph (or similarity
matrix) is called ideal (with respect to the class labels).
\end{definition}

\begin{theorem}\label{theo:ideal lrc}
Given indicator matrix $F$, if $W$ satisfies the ideal graph condition for classification,
then the row vectors of $F$ lie in the row-space of $\tilde{A}$ corresponding to
the largest singular value ($\sqrt{\beta n}$), and therefore zero fitting error is achieved: $F=D^*\tilde{A}$.
\end{theorem}

\begin{proof}
$W$ such defined is a nonnegative symmetric matrix,
hence it is a qualified similarity matrix. Because $A$ is mean-free,
the degree matrix is $S=\diag(\mb{1}^TW)=n\beta I$, and the
Laplacian matrix $L\doteq S-W$ becomes $n\beta I-\tilde{A}^T\tilde{A}$. Assume
the thin SVD \cite{golub1996matrix} of $A$ to be $U\Sigma V$, where the singular values are
arranged in descending order, then \cite{Hu2014Spectral}
\begin{equation}\label{equ:svd tildeA}
\tilde{A}=\tilde{U}\tilde{\Sigma}\tilde{V}=\begin{bmatrix}1&\\&U\end{bmatrix}\begin{bmatrix}\sqrt{\beta
n}&\\
&\Sigma\end{bmatrix}\begin{bmatrix}\frac{1}{\sqrt{n}}\mb{1}^T\\V\end{bmatrix},
\end{equation}
Further, by $L=n\beta I-\tilde{A}^T\tilde{A}$, we obtain the spectral decomposition \cite{golub1996matrix} of $L$:
\begin{equation}\label{equ:sd L}
\begin{split}
L=
\begin{bmatrix}\frac{1}{\sqrt{n}}\mb{1}&V^T&\hat{V}^T\end{bmatrix}\begin{bmatrix}0&&\\
&\beta nI-\Sigma^2&\\&&\beta
nI\end{bmatrix}\begin{bmatrix}\frac{1}{\sqrt{n}}\mb{1}^T\\V\\\hat{V}\end{bmatrix},
\end{split}
\end{equation}
where $\hat{V}$ is the complement of $\tilde{V}$.

If the condition holds, by the third
property of Laplacian matrix in Section~\ref{sec:clustering}, the
row vectors of $F$ lie in the eigenspace of $L$ with eigenvalue
zero. In view of (\ref{equ:sd L}) and (\ref{equ:svd tildeA}), the
eigenvectors of $L$ with the smallest eigenvalues are the right
singular vectors of $\tilde{A}$ with the largest singular values. Thus, the row vectors of $F$ lie in the
row-space of $\tilde{A}$ with the largest singular values (all equal $\sqrt{\beta n}$).
\end{proof}

Note that the condition ensures not only perfect reconstruction but also that the target lies in the principal row-subspace of data, or principal components (PCs) in the language of PCA \cite{jolliffe2002principal}. This is important, because the PCs correspond to stable features, whereas the minor components usually correspond to noise, especially when sampling is insufficient. In some cases, zero training error can be achieved, however, the target may be found in the minor subspace, then generalization error can be large. That is the overfitting problem. We will return to this issue in later sections.

\subsubsection{Column-space View}
Finally, we take a closer look at the vector $D^*\tilde{b}$,
and understand the voting mechanism of LRC. The facts are routine. With (\ref{equ:svd tildeA}), we have
$D^*\tilde{b}=F\tilde{V}^T
\tilde{V}_b=\sum_{i=1}^n F_i(\tilde{V}_i^T\tilde{V}_b)$, where $\tilde{V}_b=\tilde{\Sigma}^{-1}\tilde{U}^T\tilde{b}$ are the normalized full PCs of $\tilde{b}$ and
$\tilde{V}_i$ are those of
$\tilde{A}_i$. The mechanism of class prediction becomes clear: the
class of a sample is determined by the votes of the training set,
where the indicator vectors $F_i$'s play the role of the votes, and
the similarities between the training set and the sample serve as
the weights assigned to the votes. Here, the similarity is
measured by the inner product of PCs. For general data, this is not a good choice, and that is one of the limitations of
LRC. We now introduce the more powerful kernel version, which
measures the similarity based on Euclidean distance.

\subsection{Kernel Version: Normalized RBF Network (nRBFN)}\label{sec:srbf}

We will apply the kernel trick to LRC in two ways: a
traditional way leads to RBFN,\footnote{There is another standard way based on the reproducing kernel Hilbert space and representor theorem that leads to kernel classifier, see e.g., \cite{Rifkin2004In}, but it cannot lead to nRBFN.} the other way leads to nRBFN. The function matrix used by RBFN is the similarity matrix, while that used by nRBFN is the normalized Laplacian matrix. Since the similarity matrix does not share the properties of Laplacian matrix, we cannot directly analyze RBFN by spectral graph theory, whereas the link of nRBFN to the theory is straightforward. In a following subsection, we introduce the routine basis reduction to reduce the size of the networks. After that, we interpret nRBFN from the row-space and column-space views, and analyze it with spectral graph theory. An ideal working condition in the context of basis reduction is introduced, and some properties of nRBFN are shown.

\subsubsection{RBFN}\label{sec:rbf}
We derive RBFN by applying kernel trick to LRC. The solution of (\ref{equ:lrc}) can be rewritten as
$D^*=(F\tilde{A}^T(\tilde{A}\tilde{A}^T)^{-2}\tilde{A})\tilde{A}^T$,
which is a linear combination of the training data. If we assume $D=X\tilde{A}^T$, then (\ref{equ:lrc}) turns into
another objective
\begin{equation}\label{equ:F-XAA}
\min_{X}\, \|F-X\tilde{A}^T\tilde{A}\|_F^2.
\end{equation}
Applying kernel trick on $\tilde{A}^T\tilde{A}$, we get
\begin{equation}\label{equ:rbf}
\min_{X}\, \|F-XW\|_F^2.
\end{equation}
$W$ is a kernel matrix defined by some kernel function
$W_{ij}=\phi(A_i,A_j)$. In LRC case, it is a linear function
$\phi(A_i,A_j)=A_i^TA_j+\beta$. Among nonlinear ones, Gaussian kernel is most frequently used
$\phi(A_i,A_j)=\exp\{-\|A_i-A_j\|_2^2/(2\sigma^2)\}$. (\ref{equ:rbf}) is an RBFN \cite{Powell1987Radial,Broomhead1988Multivariable,Moody1989Fast},
where the kernel function is viewed as radial basis function
$\varphi(\|G_i-x\|)$:\footnote{Compared with some traditional
RBFNs, the RBFN here, derived through kernel trick, does not
append a constant vector $\mb{1}^T$ to $W$ and a bias vector to
$X$.} each column of $W$ corresponds to a sample $x$, while each row
a basis vector $G_i$. Here, the basis $G$
consists of the whole training set, $G_i=A_i$. Assume $W$ has full
rank, the solution of (\ref{equ:rbf}) is $X^*=FW^{-1}$. Given a
sample $b$, since $D^*\tilde{b}=X^*\tilde{A}^T\tilde{b}$, applying
the same trick, the class of $b$ is decided by
\begin{equation}\label{equ:rbf b}
\mathop{\arg\max}_k\,(X^*W_b)_k,
\end{equation}
where $(W_b)_i=\phi(A_i,b)$.

For RBFN, the ideal graph condition does not work, since the kernel matrix when served
as similarity matrix does not possess the same properties as its
Laplacian matrix. We are not sure whether $F$ lies in the principal subspace or minor subspace.

\subsubsection{nRBFN}
We restart the derivation with kernel trick from another way, which will lead to nRBFN.
By absorbing $\beta n$ into $X$, (\ref{equ:F-XAA}) is equivalent to
\begin{equation}\label{equ:F-XAAbeta}
\begin{aligned}
&\min_{X}\,\|F-X\tilde{A}^T\tilde{A}(\beta n)^{-1}\|_F^2
=\min_{X}\,\|F-X\tilde{A}^T\tilde{A}\,\diag(\mb{1}^T\tilde{A}^T\tilde{A})^{-1}\|_F^2.
\end{aligned}
\end{equation}
Note that
$\tilde{A}^T\tilde{A}\,\diag(\mb{1}^T\tilde{A}^T\tilde{A})^{-1}=WS^{-1}$, which is the transpose of the normalized Laplacian matrix. Applying kernel trick and
using Gaussian kernel, we obtain nRBFN:
\begin{equation}
\min_{X}\,\|F-XWS^{-1}\|_F^2.
\end{equation}
Assume $W$ has full rank, the solution is $X^*=FSW^{-1}$. Given a
sample $b$, since $X^*\tilde{A}^T\tilde{b}(\beta
n)^{-1}=X^*\tilde{A}^T\tilde{b}(\mb{1}^T\tilde{A}^T\tilde{b})^{-1}$,
applying the same kernel trick, the class of $b$ is decided by
\begin{equation}
\begin{aligned}
&\mathop{\arg\max}_k\,(X^*W_b(\mb{1}^TW_b)^{-1})_k=\mathop{\arg\max}_k\,(X^*W_b/s_b)_k,
\end{aligned}
\end{equation}
where $s_b$ is the sum of $W_b$.

Note that $WS^{-1}$ is a normalized similarity matrix. Each column
of it sums to one, so is $W_b/s_b$. nRBFN was initially mentioned by \cite{Moody1989Fast} and later derived from probability density estimation and kernel regression by \cite{Specht1990Probabilistic,Specht1991A,Xu1994On}. It is also closely related to Gaussian mixture model \cite{Tresp1993Network}. However, the underlying spectral graph background seems not yet be discovered. Before the exploration, we deal with the basis reduction problem.

\subsubsection{Basis Reduction}
In above, the bases of RBFN and nRBFN consist of the whole training set, which will lead to expensive computation. Traditionally, basis reduction is applied \cite{Que2016Back}. A smaller basis is chosen by some strategy (discussed in Section~\ref{sec:basis selection}). For the moment, we assume the basis $G=[G_1,\dots,G_r]$, $r<n$, is given. Now, $W$ is of size $r\times n$, and $W_{ij}=\phi(G_i,A_j)$. Denoting
\[
\tilde{W}\doteq WS^{-1}\in \mathbb{R}^{r\times n},
\]
the formulation of basis-reduced nRBFN becomes
\begin{equation}\label{equ:srbf}
\min_{X}\,\|F-X\tilde{W}\|_F^2.
\end{equation}

nRBFN and RBFN are special cases of linear regression, with sample vectors
replaced by similarity vectors. Assume $W$ is of full rank, the solution of nRBFN is
$X^*=F\tilde{W}^T(\tilde{W}\tilde{W}^T)^{-1}$ (that of RBFN is
$X^*=FW^T(WW^T)^{-1}$). Given a
sample $b$, in nRBFN its class is decided by
\begin{equation}\label{equ:srbf b}
\mathop{\arg\max}_k\,(X^*\tilde{W}_b)_k,
\end{equation}
where $\tilde{W}_b\doteq W_b/s_b$, $W_b$ are the similarities between $b$ and the basis, and $s_b$ is the sum of $W_b$.

Hereafter, we focus on basis-reduced nRBFN.

\subsubsection{Row-space View and Column-space View to nRBFN}\label{sec:row-col srbf}
We will show the spectral graph theory underlying nRBFN, and introduce some basic properties as well as interpretations concerning nRBFN.

1. From column-space view, we will show that, as LRC, the class prediction of nRBFN is also via voting mechanism. First, besides $\tilde{W}$ has a probability
interpretation, the weight matrix $X^*$ also has a quasi-probability
interpretation. We have
\begin{lemma}\label{lem:1X}
Each column of the weight matrix $X^*$ sums to
one: $\mb{1}^TX^*=\mb{1}^TF\tilde{W}^T(\tilde{W}\tilde{W}^T)^{-1}=\mb{1}^T$.
\end{lemma}

\begin{proof}
The row-space of $\tilde{W}$
contains $\mb{1}^T\in \mathbb{R}^{1\times n}$, because
$\mb{1}^T\tilde{W}=\mb{1}^T$. Thus, from the projection point of
view,
$\mb{1}^T\tilde{W}^T(\tilde{W}\tilde{W}^T)^{-1}\tilde{W}=\mb{1}^T$.
By this, $\mb{1}^TX^*\tilde{W}=\mb{1}^T$. On the other hand, since
$\tilde{W}$ has full rank, the solution of $x\tilde{W}=\mb{1}^T$ is
unique. However, both $\mb{1}^T\in \mathbb{R}^{1\times r}$ and
$\mb{1}^TX^*$ are the solutions, so we conclude
$\mb{1}^TX^*=\mb{1}^T$.
\end{proof}
Now considering (\ref{equ:srbf b}), the
class of sample $b$ is decided by the voting of the basis. Each
basis vector keeps a vote $X^*_i$, and the weight assigned to the
vote is the normalized similarity of $b$ to that basis vector $G_i$.
In contrast to LRC, the votes are not indicator vectors, and the
weights are not computed by inner product. Rather, the votes of the
basis are gathered from another voting of the training data,
$X^*_i=F(\tilde{W}^T(\tilde{W}\tilde{W}^T)^{-1})_i$. Note that
$\mb{1}^T\tilde{W}^T(\tilde{W}\tilde{W}^T)^{-1}=\mb{1}^T$, since
$\mb{1}^TX^*=\mb{1}^T$. We can expect that when a basis vector is
more ``reliable'', e.g., lying in the center of a class, the vote it
keeps would concentrate in its class, whereas when lying in
the overlapping region, the vote would be distributed more evenly. Generally, except under the ideal condition below, $X^*$ will include negative value, which represents objection.

2. From row-space view, similar to LRC, nRBFN finds the closest subspace in the row-space of $\tilde{W}$ to approximate $F$. Again, an ideal condition ensuring perfect reconstruction exists. However, in the context of basis reduction, the graph should be generalized to a bipartite graph: one side of the vertices consists of the basis, the other side consists of the training set. We should assume the basis is a subset of the training set, and each class is represented by at least one basis vector. The original case where the basis consists of the whole training set is a special case of bipartite graph. Denoting $F_{G_i}$ to be the indicator vector of $G_i$, the condition is as follows:
\begin{definition}\label{def:ideal classification2}
\tb{(ideal bipartite-graph condition for classification)} If the weights
between basis vertices and data vertices of different classes are all zero, i.e., $\forall
i,j$, $W_{ij}=0$ if $F_{G_i}\neq F_j$, then the bipartite-graph (or similarity
matrix) is called ideal (with respect to the class labels).
\end{definition}
With these prerequisites, we have the following theorem indicating zero fitting error:
\begin{theorem}\label{theo:ideal srbf}
Given indicator matrix $F$, if $W$ satisfies the ideal bipartite-graph condition for classification, then the row vectors of $F$ lie in the row-space
of $\tilde{W}$, and zero fitting error is achieved: $F=X^*\tilde{W}$.
\end{theorem}
The proof of the theorem is manifest: the rows of $\tilde{W}$ corresponding to the same class sum to an indicator vector of that class.

When the condition holds, it can be proved that the votes $X^*$ become an indicator matrix. In this case the votes are ideal and very confident, since they concentrate in one class.
\begin{proposition}\label{pro:X indicator}
If the ideal bipartite-graph condition for classification holds, $X^*$ becomes an indicator matrix. For basis vector $G_i$, $X^*_i=F_{G_i}$.
\end{proposition}

\begin{proof}
Without loss of generality, assume the training samples of the same class are arranged consecutively, then $\tilde{W}$ is block-diagonal, so is $(\tilde{W}\tilde{W}^T)^{-1}$. Further, $\tilde{W}^T(\tilde{W}\tilde{W}^T)^{-1}$ has the same nonzero blocks as $\tilde{W}^T$. Since $X^*_i=F(\tilde{W}^T(\tilde{W}\tilde{W}^T)^{-1})_i$, we see that $X^*_i$ is a linear combination of the indicator vectors of the same class as $G_i$, which means there is only one nonzero in $X^*_i$. By Lemma~\ref{lem:1X}, we conclude that the nonzero value is one. Therefore, $X^*_i$ is an indicator vector of $G_i$.
\end{proof}
When the condition does not hold exactly, that is when the classes have some overlapping but not heavy, negative values may present in $X^*$, but their magnitude should be small, because $X^*$ is a continuous function of $W$.

Unlike LRC case, whether $F$ lies in the leading row-subspace is less obvious. When the basis consists of the whole training set and the ideal condition is satisfied, by property of the normalized Laplacian matrix (Section~\ref{sec:clustering}),
rows of $F$ are the left eigenvectors with the largest
\emph{eigenvalues}. We expect they are close to the leading row-subspace that corresponds to the largest \emph{singular values}. For the general bipartite-graph case, please refer to Appendix~\ref{sec:gap} for detailed investigation. We present the main result below.
\begin{theorem}\label{theo:F leading}
Under the ideal bipartite-graph condition for classification, the row vectors of $F$ become the right singular vectors of $\tilde{W}$ corresponding to the largest singular values, if and only if the row sums of $\tilde{W}$ are even within each class: $\sum_k \tilde{W}_{ik}=\sum_k \tilde{W}_{jk}$, $\forall i,j, F_{G_i}=F_{G_j}$.
\end{theorem}
The theorem suggests that when designing the model, a balanced system, which means the row sums of $\tilde{W}$ are as even as possible within each class, is preferred.

Hereafter, without confusion, we will simply refer the above three conditions as \emph{ideal graph condition}.

\section{Regularization for LRC and nRBFN}\label{sec:regularization}
In practice, the ideal graph conditions are not easy to meet exactly. In this case, zero fitting error may not be achieved, and part of the found closest components may lie in the minor subspace. This will lead to the increment of overfitting risk. In this section, we introduce the traditional $\ell_2$-norm regularization, and qualitatively show its effect on controlling the overfitting risk from the spectral view. The regularized versions of LRC and nRBFN are what we will really apply in real world.
\subsection{Regularized LRC}\label{sec:r lrc}
For LRC, the ideal graph condition is hard to meet, except perhaps for high-dimensional data: by the construction of
$W$, the condition essentially requires that, after translating along
a new dimension, different classes become orthogonal. Thus in general case, the leading
row-subspace of $\tilde{A}$ may deviate much from $F$. LRC then searches the entire
row-space to find a closest subspace to approximate $F$. The found
subspace may correspond to small singular values, which may
represent discriminative features or, more frequently, noise (e.g., due to
insufficient sampling). In other words, LRC may deem the noisy
components of data as the discriminative features for
classification. Poor generalization ability can be expected.

In this regard, we would like to encourage LRC to search within the
principal row-subspace. This can be achieved via $\ell_2$-norm regularization:
\begin{equation}\label{equ:lrc lambda}
\min_{D}\, \|F-D\tilde{A}\|_F^2+\lambda '\|D\|_F^2,
\end{equation}
where $\lambda '>0$ is a scalar weight. The unique closed-from solution is
$D^*=F\tilde{A}^T(\tilde{A}\tilde{A}^T+\lambda 'I)^{-1}$, regardless
of the rank of $A$. The label of a new sample is decided as
(\ref{equ:lrc b}).

In this case, the reconstruction of the training set is
$D^*\tilde{A}=F\tilde{A}^T(\tilde{A}\tilde{A}^T+\lambda
'I)^{-1}\tilde{A}$. Assume the SVD of $\tilde{A}$ to be
$\tilde{U}\tilde{\Sigma}\tilde{V}$, then $\tilde{A}^T(\tilde{A}\tilde{A}^T+\lambda
'I)^{-1}\tilde{A}=\tilde{V}^T\Lambda \tilde{V}$, where $\Lambda$ is
a diagonal matrix with
$\Lambda_{ii}=\tilde{\sigma}_i^2/(\tilde{\sigma}_i^2+\lambda ')<1$.
For those large singular values $\tilde{\sigma}_i\gg \lambda '$,
$\Lambda_{ii}\approx 1$ so the principal subspace is preserved,
while for those small ones $\tilde{\sigma}_i\ll \lambda '$,
$\Lambda_{ii}\approx 0$ so the minor components are
suppressed. In implementation, for the ease of setting $\lambda '$, we rewrite
it to be $\lambda '=\lambda\|\tilde{A}\|_F^2$. Note that
$\|\tilde{A}\|_F^2=\sum_{i} \tilde{\sigma}_i^2$, so the contrast between $\tilde{\sigma}_i^2$ and $\lambda'$ is easier to control. Back to the
objective, $\|F-D^*\tilde{A}\|_F^2=\|F-F\tilde{V}^T\Lambda
\tilde{V}\|_F^2$. Now LRC projects $F$ onto the principal
row-subspace and reconstructs. The reconstruction error may increase
slightly, but overfitting is alleviated. A quantitative analysis will be conducted in Section~\ref{sec:risk}.

\subsection{Regularized nRBFN}
The ideal graph condition for nRBFN does not require orthogonality between the classes. Nevertheless, when the condition is not exactly met, to prevent nRBFN seeking the closest components in the
minor subspace, we introduce regularization (the analysis follows LRC, and we omit):
\begin{equation}\label{equ:srbf lambda}
\min_{X}\, \|F-X\tilde{W}\|_F^2+\lambda ' \|X\|_F^2,
\end{equation}
where $\lambda'=\lambda\|\tilde{W}\|_F^2$, and $\lambda>0$ is a scalar
weight. The solution becomes
$X^*=F\tilde{W}^T(\tilde{W}\tilde{W}^T+\lambda 'I)^{-1}$, regardless
of the rank of $\tilde{W}$. The label of a sample $b$ is decided as (\ref{equ:srbf b}).

\section{Error and Risk Analysis}\label{sec:risk}
The tradeoff between fitting error and overfitting risk is an important problem of classification. In this section, from the spectral point of view, we quantitatively analyze this problem for LRC and nRBFN, and reveal the effect of $\ell_2$-regularization further. First, we define a quantitative criterion, the spectral risk, for the measurement of overfitting risk. The fitting error is also formally defined. Next, we analyze the un-regularized nRBFN and LRC in mainly ideal cases. The bounds of the two quantities will be derived. Finally, we investigate the $\ell_2$-regularization (independent of the ideal graph condition). We will show that the two terms in the $\ell_2$-regularized objective are one-one correspondent to the fitting error and spectral risk, and study the effect of $\ell_2$-regularization on trading off the error and risk.

\subsection{Definitions of Spectral Risk and Fitting Error}\label{sec:definition error and risk}
The definitions apply to linear regression, including LRC and nRBFN as special cases. We will define an absolute measure and a relative measure for both spectral risk and fitting error, the reasons will be clear later. The relative measures will be used as default definitions.

Let the linear regression problem be formulated as
\begin{equation}\label{equ:lr}
\min_{D}\, \|F-DA\|_F^2,
\end{equation}
where $F\in \mathbb{R}^{K\times n}$ is any target matrix not limited to indicators, $A\in \mathbb{R}^{r\times n}$ ($r\leq n$) is any data matrix of full rank. The solution is $D^*=FA^T(AA^T)^{-1}$. To exclude meaningless case, we assume $D^*\neq \mb{0}$, which means $FA^T\neq \mb{0}$, i.e., the data is not orthogonal to the target.

\subsubsection{Spectral Risk}
The spectral risk measures the deviation of the found components to the principal subspace. First, we define an absolute measure.
\begin{definition}\label{def:absrisk}
\tb{(absolute spectral risk)} The absolute spectral risk is defined as
\begin{equation}
\alpha \doteq \|D^*\|_F^2.
\end{equation}
\end{definition}
The justification can be understood by the following spectral expression. Assume the SVD of A to be $A=U\Sigma V$, then for problem (\ref{equ:lr}), $D^*=FV^T\Sigma^{-1}U^T$, and we have
\begin{proposition}
For linear regression problem (\ref{equ:lr}),
\begin{equation}\label{equ:absrisk}
\alpha=\sum_{i=1}^r\frac{a_i^2}{\sigma_i^2},
\end{equation}
where $a_i^2$ is the projection of the target onto $V_i$: $a_i^2\doteq \sum_{k=1}^K(F_kV_i^T)^2$, and $F_k$ is the $k$th row of $F$.
\end{proposition}
It implies that if the projections concentrate in the leading singular vectors, that is the closest components lie in the principal subspace, $\alpha$ will be small. Conversely, if they concentrate in the rear singular vectors, $\alpha$ will be large. Thus, as an absolute measure, $\|D^*\|_F^2$ is reasonable.

However, a meaningful range of the absolute measure cannot be determined. In order to cancel out the volume of data so that the measures between different data of the same model can be compared, we now define a relative measure by normalizing the projections and singular values.
\begin{definition}\label{def:risk}
\tb{(spectral risk)} The relative spectral risk, simply called spectral risk, is defined as
\begin{equation}
\gamma \doteq \frac{\|D^*\|_F^2\|A\|_F^2}{\|D^*A\|_F^2}.
\end{equation}
\end{definition}
For problem (\ref{equ:lr}), by noting $\|A\|_F^2=\sum_{i=1}^r \sigma_i^2$ and $\|D^*A\|_F^2=\|FV^TV\|_F^2=\sum_{i=1}^r a_i^2$, we have
\begin{proposition}
For linear regression problem (\ref{equ:lr}),
\begin{equation}\label{equ:risk}
\gamma=\sum_{i=1}^r\frac{\tilde{a}_i^2}{\tilde{\sigma}_i^2},
\end{equation}
where $\tilde{a}_i^2\doteq a_i^2/\sum_{i=1}^r a_i^2$ and $\tilde{\sigma}_i^2\doteq \sigma_i^2/\sum_{i=1}^r \sigma_i^2$.
\end{proposition}

We can easily obtain the following range and bounds of $\gamma$:
\begin{proposition}\label{pro:risk bounds}
The range of the spectral risk is $\gamma\geq 1$, and
\begin{equation}\label{equ:bound sigma}
1\leq \frac{1}{\tilde{\sigma}_1^2}\leq \gamma \leq \frac{1}{\tilde{\sigma}_r^2}.
\end{equation}
The minimum value 1 is achieved, if and only if the data dimension is one, i.e., $r=1$.
\end{proposition}
The minimum bound implies that when we use the smallest basis having only one vector, the smallest risk is achieved. However, in this case, the fitting error can be quite large. There is a tradeoff between the fitting error and the spectral risk.

\subsubsection{Fitting Error}
First of all, it should be made clear that the fitting error is distinct from the error rate of classification. The absolute measure of fitting error is defined straightforwardly:
\begin{definition}\label{def:abserror}
\tb{(absolute fitting error)} The absolute fitting error is defined as
\begin{equation}
f \doteq \|F-D^*A\|_F^2.
\end{equation}
\end{definition}
For problem (\ref{equ:lr}), we have $f=\|F\|_F^2-\|D^*A\|_F^2=n-\sum_{i=1}^r a_i^2$. For reason that will be clear later, we define the relative measure to be:
\begin{definition}\label{def:error}
\tb{(fitting error)}
The relative fitting error, simply called fitting error, is defined as
\begin{equation}
\epsilon\doteq \frac{\|F-D^*A\|_F^2}{\|D^*A\|_F^2}+1.
\end{equation}
\end{definition}
For problem (\ref{equ:lr}), the definition is equivalent to $\epsilon\doteq \|F\|_F^2/\|D^*A\|_F^2$. It is easy to see that
\begin{proposition}\label{pro:error}
For linear regression problem (\ref{equ:lr}), $\epsilon=n/\sum_{i=1}^r a_i^2$.
Its range is $\epsilon\geq 1$. The minimum value 1 is achieved if and only if the row-space of data completely covers the target, i.e., $D^*A=A$.
\end{proposition}

\subsection{Error and Risk Bounds of nRBFN}\label{sec:risk srbf}
The ideal graph condition ensures zero fitting error and low overfitting risk, we will calculate the specific bounds for nRBFN (un-regularized), and then extend to the perturbation case where the condition is not satisfied exactly. We assume $F\tilde{W}^T\neq \mb{0}$, so that $X^*\neq \mb{0}$.
\subsubsection{Ideal Case}
Let $r_k$ denote the size of basis for the $k$th class, and $n_k$ the number of training samples of that class. We have the following result:
\begin{theorem}\label{theo:nrbfn ideal risk}
For nRBFN problem (\ref{equ:srbf}), when the ideal graph condition is satisfied, the fitting error achieves the minimum value, $\epsilon=1$, and the spectral risk has the bounds
\begin{equation}\label{equ:srbf gamma}
\frac{r}{n}\sum_{k=1}^K \frac{n_k}{r_k}\leq \gamma \leq r.
\end{equation}
The maximum risk is achieved when there is only one nonzero entry in each column of $\tilde{W}$. The minimum risk is approached when the entries in each column approach distributing uniformly within the corresponding class, it is achieved if and only if $r_k=1$ for all $k$.
\end{theorem}
\begin{proof}
By Theorem~\ref{theo:ideal srbf}, when the condition holds, perfect reconstruction is achieved, i.e., $X^*\tilde{W}=F$, so $\|X^*\tilde{W}\|_F^2=n$, $f=1$. Besides, by Proposition~\ref{pro:X indicator}, $X^*$ becomes an indicator matrix, so $\|X^*\|_F^2=r$. The spectral risk then equals $r\|\tilde{W}\|_F^2/n$, depending only on $\|\tilde{W}\|_F^2$. Recall that each column of $\tilde{W}$ sums to one. For a vector $x\in \mathbb{R}^{p}$ with $\|x\|_1=1$, $1/\sqrt{p}\leq \|x\|_2 \leq 1$, the upper bound is obtained when there is only one nonzero entry in $x$, while the lower bound is obtained when the entries distribute uniformly, i.e., $x_i=1/p$. Therefore, we get $\sum_{k=1}^K n_k/r_k\leq\|\tilde{W}\|_F^2\leq n$. The lower bound cannot be reached except when $r_k=1$ for all $k$, otherwise, the rank of $\tilde{W}$ will not be full, violating the assumption.
\end{proof}

\begin{corollary}
If both $n_k$ and $r_k$ are uniform among the classes, i.e., $n_k=n/K$, $r_k=r/K$, for all $k$, (\ref{equ:srbf gamma}) becomes
\begin{equation}
K\leq \gamma \leq r.
\end{equation}
\end{corollary}

The corollary conveys clear implications. It tells that in the ideal case the spectral risk is lower bounded by the number of classes--a quantity representing ``problem complexity'', and upper bounded by the size of basis--a quantity representing ``model complexity''. First, the lower bound implies that however the model is the spectral risk will never be lower than the problem complexity. As the number of classes increases, the spectral risk increases too. By Theorem~\ref{theo:nrbfn ideal risk}, the risk approaches the minimum when the system is balanced, i.e., when the entries are uniform. This is a stronger condition than the uniform row-sum condition in Theorem~\ref{theo:F leading}. Although they analyze the deviation of $F$ to the principal subspace by different manners, Theorem~\ref{theo:nrbfn ideal risk} by divisive manner while Theorem~\ref{theo:F leading} by subtractive manner, the results are consistent. They both prefer a balanced system. Second, the upper bound implies a tradeoff between the fitting error and the overfitting risk: larger basis means lower error but higher risk. In addition, the size of basis also makes us recall the VC-dimension \cite{Vapnik2000The}--a classical measure of model complexity. Although the definitions are different, it happens that the VC-dimension of nRBFN (assuming the basis size is fixed, not a parameter) is also $r$. This shows some coincidence between the two concepts. Detailed comparisons are beyond the scope of the paper and we omit.

Although the conclusions hold for the ideal case, the theorem provides a foundation for the analysis of noisy case. We now explore.

\subsubsection{Perturbation Case}\label{sec:risk srbf perturbation}
We can deem the noisy case as perturbed from an ideal case, and then analyze it with matrix perturbation theory. As will be shown the noise is required to be only tiny.

Denote the noisy normalized similarity matrix to be $\tilde{W'}$, it can be decomposed as $\tilde{W'}=\tilde{W}+\Delta W$, where $\tilde{W}$ is an ideal normalized similarity matrix and $\Delta W$ is noise. Given $\tilde{W'}$, for the purpose of perturbation analysis, we do not need to know the true $\tilde{W}$. Constructing one will suffice. Among many choices, the simplest one is as follows: set the between-class entries of $\tilde{W'}$ to zero, and then normalize each column, this leads to a qualified $\tilde{W}$. Subsequently, the noise is determined, $\Delta W=\tilde{W'}-\tilde{W}$. It can be proved that, among all the choices of $\tilde{W}$, the noise induced by this manner is the minimum in $\ell_1$-norm sense. Details are omitted.

Before presenting the results, we introduce some notations. Let
\[
\xi\doteq \|\tilde{W}^\dag\|_2\|\Delta W\|_2,
\]
where $\|\cdot\|_2$ for a matrix denotes the spectral norm, $\tilde{W}^\dag$ denotes the pseudo-inverse of $\tilde{W}$, which is equivalent to $\tilde{W}^T(\tilde{W}\tilde{W}^T)^{-1}$ in our case. Let
\[
\delta\doteq \|\Delta W\|_F/\|\tilde{W}\|_F,
\]
\[
n_\rho\doteq \sqrt{n_p/n_q},\quad\quad r_\rho\doteq \sqrt{r_a/r_b},
\]
where $n_p=\max_k n_k$, $n_q=\min_k n_k$, $r_a=\max_k r_k$, and $r_b=\min_k r_k$. Finally, let $\epsilon'$ and $\gamma'$ be the fitting error and spectral risk of the noisy case respectively, and $\gamma$ the spectral risk of the case $\tilde{W}$. We have the following bounds for $\epsilon'$ and $\gamma'$.
\begin{theorem}\label{theo:srbf pert risk}
For nRBFN problem (\ref{equ:srbf}), assume $\tilde{W'}=\tilde{W}+\Delta W$, and $\tilde{W}$ is of full rank satisfying the ideal graph condition, if
\begin{equation}\label{equ:srbf cond}
\xi<1/n_\rho,
\end{equation}
then
\begin{equation}\label{equ:srbf noisy error}
\epsilon'\leq \frac{n_\rho^2\xi^2}{(1- n_\rho\xi)^2}+1,
\end{equation}
\begin{equation}\label{equ:srbf noisy gamma}
\gamma'\leq \gamma(1+\delta)^2\Big{(}\frac{1+  r_\rho \xi}{1- n_\rho \xi}\Big{)}^2.
\end{equation}
\end{theorem}

\begin{proof}
Denote the solution of the noisy case to be $X'$, and the ideal case $X$. For simplicity, denote $X'\tilde{W'}$ by $Y'$, and $X\tilde{W}$ by $Y$. Let the difference denote by $\Delta$, e.g., $\Delta X=X'-X$, $\Delta Y=Y'-Y$. We will apply the following results from matrix perturbation theory (\cite{David1997Numerical} Theorem 18.1):
\[
\frac{\|\Delta y\|_2/\|y\|_2}{\|\Delta W\|_2/\|\tilde{W}\|_2}\leq \frac{\kappa}{\cos \theta},
\quad\quad
\frac{\|\Delta x\|_2/\|x\|_2}{\|\Delta W\|_2/\|\tilde{W}\|_2}\leq \kappa+\frac{\kappa^2\tan \theta}{\eta},
\]
where $x$ is any row of $X$, and $y$ is the corresponding row of $Y$. $\kappa$ is the condition number of $\tilde{W}$, $\kappa\doteq \|\tilde{W}\|_2\|\tilde{W}^\dag\|_2$. $\theta$ is the included angle between the target vector and its reconstruction $y$. $\eta\doteq \|x\|_2\|\tilde{W}\|_2/\|y\|_2$. In our context, $\tilde{W}$ satisfies the ideal condition, therefore $\theta=0$. The above results can be reduced to much simpler forms:
\[
\frac{\|\Delta y\|_2/\|y\|_2}{\|\Delta W\|_2/\|\tilde{W}\|_2}\leq \kappa,
\quad\quad
\frac{\|\Delta x\|_2/\|x\|_2}{\|\Delta W\|_2/\|\tilde{W}\|_2}\leq \kappa.
\]
By the definitions of $\kappa$ and $\delta$, they lead to
\begin{equation}\label{equ:dy}
\|\Delta y\|_2/\|y\|_2\leq \xi,
\end{equation}
\begin{equation}\label{equ:dx}
\|\Delta x\|_2/\|x\|_2\leq \xi.
\end{equation}
We now begin the proof.

First, we consider $\|Y'\|_F$, since both $\epsilon'$ and $\gamma'$ involve this denominator.
\begin{equation}\label{equ:Y'}
\|Y'\|_F\geq \|Y\|_F-\|\Delta Y\|_F=\|Y\|_F(1-\|\Delta Y\|_F/\|Y\|_F).
\end{equation}
To apply (\ref{equ:dy}), we have to convert the Frobenius norm of $Y$ to the $\ell_2$ norm of its rows $Y_k$'s.
\[
\begin{split}
\|\Delta Y\|_F/\|Y\|_F&=\sqrt{\frac{\sum_k \|\Delta Y_k\|_2^2}{\sum_k \|Y_k\|_2^2}}\\
&\leq \sqrt{\frac{K\|\Delta Y_u\|_2^2}{K\|Y_{l}\|_2^2}}\\
&= \|\Delta Y_u\|_2/\|Y_{l}\|_2\\
&=(\|Y_{u}\|_2/\|Y_{l}\|_2)(\|\Delta Y_{u}\|_2/\|Y_{u}\|_2),
\end{split}
\]
where subscript $u=\arg\max_{k} \|\Delta Y_k\|_2$, and $l=\arg\min_{k} \|Y_k\|_2$. Note that $\tilde{W}$ satisfies the ideal condition, by Theorem~\ref{theo:ideal srbf}, $Y=F$, and we have
\begin{equation}\label{equ:dy/y}
\begin{split}
\|\Delta Y\|_F/\|Y\|_F&\leq (\|F_{u}\|_2/\|F_{l}\|_2)(\|\Delta Y_{u}\|_2/\|Y_{u}\|_2)\\
&\leq \sqrt{n_{p}/n_{q}}(\|\Delta Y_{u}\|_2/\|Y_{u}\|_2)\\
&=n_\rho (\|\Delta Y_{u}\|_2/\|Y_{u}\|_2)\\
&\leq n_\rho \xi,
\end{split}
\end{equation}
where the last line invokes (\ref{equ:dy}). Substituting (\ref{equ:dy/y}) into (\ref{equ:Y'}), we obtain
\begin{equation}\label{equ:Y'2}
\|Y'\|_F\geq \|Y\|_F(1- n_\rho\xi).
\end{equation}
By the assumption, $\xi<1/n_\rho$, we are sure $1- n_\rho\xi>0$ so that $\|Y'\|_F> 0$.

Second, using (\ref{equ:Y'2}) and then (\ref{equ:dy/y}), the fitting error can be estimated
\[
\begin{split}
\epsilon' =\frac{\|F-Y'\|_F^2}{\|Y'\|_F^2}+1
\leq \frac{\|F-Y-\Delta Y\|_F^2}{\|Y\|_F^2(1- n_\rho\xi)^2}+1
= \frac{\|\Delta Y\|_F^2}{\|Y\|_F^2(1- n_\rho\xi)^2}+1
=\frac{n_\rho^2\xi^2}{(1- n_\rho\xi)^2}+1.
\end{split}
\]

Third, we consider $\|X'\|_F$.
\begin{equation}\label{equ:X'}
\|X'\|_F\leq \|X\|_F+\|\Delta X\|_F=\|X\|_F(1+\|\Delta X\|_F/\|X\|_F).
\end{equation}
By Proposition~\ref{pro:X indicator}, $X$ is an indicator matrix. Following the skill of the case $\|\Delta Y\|_F/\|Y\|_F$, we get
\begin{equation}\label{equ:dx/x}
\begin{split}
\|\Delta X\|_F/\|X\|_F&\leq (\|X_{u'}\|_2/\|X_{l'}\|_2)(\|\Delta X_{u'}\|_2/\|X_{u'}\|_2)\\
&\leq \sqrt{r_a/r_b}(\|\Delta X_{u'}\|_2/\|X_{u'}\|_2)\\
&=r_\rho (\|\Delta X_{u'}\|_2/\|X_{u'}\|_2)\\
&\leq r_\rho \xi.
\end{split}
\end{equation}
where the last line invokes (\ref{equ:dx}). Substituting (\ref{equ:dx/x}) into (\ref{equ:X'}), we obtain
\begin{equation}\label{equ:X'2}
\|X'\|_F\leq \|X\|_F(1+ r_\rho\xi).
\end{equation}

Forth, using (\ref{equ:Y'2}) and (\ref{equ:X'2}), we can get the bound of $\sqrt{\gamma'}$.
\[
\begin{split}
\sqrt{\gamma'}&=\|\tilde{W'}\|_F\frac{\|X'\|_F}{\|Y'\|_F}\\
&\leq \|\tilde{W'}\|_F\frac{\|X\|_F(1+ r_\rho\xi)}{\|Y\|_F(1- n_\rho\xi)}\\
&=\frac{\|X\|_F\|\tilde{W}\|_F}{\|Y\|_F}\frac{\|\tilde{W'}\|_F}{\|\tilde{W}\|_F}\Big{(}\frac{1+ r_\rho\xi}{1- n_\rho\xi}\Big{)}\\
&\leq \sqrt{\gamma}\frac{\|\tilde{W}\|_F+\|\Delta W\|_F}{\|\tilde{W}\|_F}\Big{(}\frac{1+ r_\rho\xi}{1- n_\rho\xi}\Big{)}\\
&=\sqrt{\gamma}(1+\delta)\Big{(}\frac{1+ r_\rho\xi}{1- n_\rho\xi}\Big{)}.
\end{split}
\]
Finally, (\ref{equ:srbf noisy gamma}) is obtained.
\end{proof}

Moreover, we have the following two corollaries.
\begin{corollary}\label{coro:srbf risk}
$\delta\leq \xi$, and
\begin{equation}\label{equ:srbf noisy gamma2}
\gamma'\leq \gamma(1+\delta)^2\Big{(}\frac{1+  r_\rho \xi}{1- n_\rho \xi}\Big{)}^2\leq \gamma\frac{(1+r_\rho \xi)^4}{(1-n_\rho \xi)^2}.
\end{equation}
\end{corollary}

\begin{proof}
Similar to the relation between the Frobenius norm and the $\ell_2$ norm in the proof of Theorem~\ref{theo:srbf pert risk}, we can relate $\delta$ to $\xi$. Let the singular values of $\Delta W$ and $\tilde{W}$ in descending order to be $\tau_i$'s and $\sigma_i$'s respectively. By assumption, $\tilde{W}$ is of full rank, so $\sigma_r\neq 0$. Note that $\|\tilde{W}^\dag\|_2=1/\sigma_r$. We have
\[
\delta= \|\Delta W\|_F/\|\tilde{W}\|_F=\sqrt{\frac{\sum_i \tau_i^2}{\sum_i \sigma_i^2}}\leq \tau_1/\sigma_r=\|\Delta W\|_2\|\tilde{W}^\dag\|_2=\xi.
\]
Further, $r_\rho\geq 1$, so $\delta\leq r_\rho\xi$, and (\ref{equ:srbf noisy gamma2}) is obtained.
\end{proof}

\begin{corollary}
If the classes and basis are even, i.e., $n_\rho=1$, $r_\rho=1$, then the condition (\ref{equ:srbf cond}) becomes $\xi<1$, and (\ref{equ:srbf noisy error}), (\ref{equ:srbf noisy gamma}) become
\begin{equation}
\epsilon'\leq \frac{\xi^2}{(1- \xi)^2}+1,
\end{equation}
\begin{equation}
\gamma'\leq \gamma(1+\delta)^2\Big{(}1+\frac{2\xi}{1-\xi}\Big{)}^2\leq \gamma\frac{(1+\xi)^4}{(1-\xi)^2}.
\end{equation}
\end{corollary}

The above results suggest that, comparing with the ideal case, the fitting error and spectral risk increase by a factor of $n_\rho^2\xi^2/(1- n_\rho\xi)^2+1\geq 1$ and $(1+r_\rho \xi)^4/(1-n_\rho \xi)^2\geq 1$ respectively. When $\xi\rightarrow 0$, they approach 1. $\xi<1$ is a condition frequently appeared in the perturbation analysis of linear system \cite{Stewart1990Matrix}. We have reproduced it in our context (stemming from (\ref{equ:Y'2})). $\xi$ measures the noise magnitude $\|\Delta W\|_2$ (i.e., the largest singular value of $\Delta W$) relative to the smallest singular value of $\tilde{W}$. In practice, $\xi<1$ requires the noise to be only tiny.

\subsection{Error and Risk Bounds of LRC}\label{sec:risk lrc}
We derive the bounds for un-regularized LRC in the ideal case. The results are not as significant as those of nRBFN, so less emphasis is put. Denote $\bar{\zeta}\doteq \|A\|_F^2/n$ to be the mean squared length of original data, and $\zeta_\rho\doteq \max_i \|A_i\|_2^2/\min_i \|A_i\|_2^2$ to be the ratio between the maximum and minimum length, and $\theta_u$ to be the maximum included angle between original data pairs. We have
\begin{theorem}
For LRC problem (\ref{equ:lrc}), when the ideal graph condition is satisfied, the fitting error achieves the minimum value, $\epsilon=1$, and the spectral risk
\begin{equation}\label{equ:lrc gamma}
\gamma=1+\frac{\bar{\zeta}}{\beta},
\end{equation}
\begin{equation}\label{equ:lrc gamma bound}
1+\frac{1}{\zeta_\rho |\cos \theta_u|}\leq \gamma \leq 1+\frac{\zeta_\rho}{|\cos \theta_u|}.
\end{equation}
\end{theorem}

\begin{proof}
We calculate $\gamma$ using the spectral expression $\gamma=\sum_{i=1}^r \tilde{a}_i^2/\tilde{\sigma}_i^2$. By Theorem~\ref{theo:ideal lrc}, rows of $F$ are the leading singular vectors of $\tilde{A}$, corresponding to singular value $\sqrt{\beta n}$. It implies $\tilde{a}_1^2+\dots+\tilde{a}_K^2=1$ and $\tilde{a}_i=0$ for all $i>K+1$, and $\tilde{\sigma}_1^2=\dots=\tilde{\sigma}_K^2$. Thus, we have $\gamma=1/\tilde{\sigma}_1^2$. Since $\tilde{\sigma}_1^2=\beta n/\|\tilde{A}\|_F^2=\beta n/(\beta n+\|A\|_F^2)$, (\ref{equ:lrc gamma}) is obtained. Next, by the definition of $\beta$ (\ref{equ:beta}), assume $A_{i^*}$, $A_{j^*}$ achieve $\min_{i,j} A_i^TA_j$, then $\beta=-\|A_{i^*}\|_2\|A_{j^*}\|_2\cos \theta_u$, where $\theta_u>\pi/2$. Hence, $\bar{\zeta}/\beta=\bar{\zeta}/(\|A_{i^*}\|_2\|A_{j^*}\|_2|\cos \theta_u|)$, and (\ref{equ:lrc gamma bound}) is easy to verify.
\end{proof}

\subsection{Effect of $\ell_2$-norm Regularization}
First, we extend the previous definitions of fitting error and spectral risk to regularized linear regression, then we show the one-one correspondence between the two quantities and the regularized objective, and finally we study the effect of the regularization.

Let the regularized linear regression problem be
\begin{equation}\label{equ:rlr}
g(\tilde{D})=\min_{\tilde{D}}\, \|F-\tilde{D}A\|_F^2+\lambda' \|\tilde{D}\|_F^2,
\end{equation}
where $\lambda '>0$.

The previous four definitions are extended trivially by replacing with the new $\tilde{D}^*$ of problem (\ref{equ:rlr}). The spectral expressions are changed to be:
\begin{proposition}\label{pro:rlr}
For regularized linear regression problem (\ref{equ:rlr}),
\[
\alpha=\sum_{i=1}^r \frac{a_i^2\sigma_i^2}{(\sigma_i^2+\lambda')^2};\quad\quad
\gamma=\sum_{i=1}^r\frac{\tilde{a'}_i^2}{\tilde{\sigma}_i^2},
\]
where ${a'}_i^2\doteq a_i^2\sigma_i^4/(\sigma_i^2+\lambda')^2$, and $\tilde{a'}_i^2\doteq {a'}_i^2/\sum_i {a'}_i^2 $;
\[
f=n-\sum_{i=1}^r a_i^2 (1- \frac{{\lambda'}^2}{(\sigma_i^2+\lambda')^2});\quad\quad
\epsilon=\frac{n-\sum_i 2a_i^2\sigma_i^2\lambda'/(\sigma_i^2+\lambda')^2}{\sum_i a_i^2\sigma_i^4/(\sigma_i^2+\lambda')^2}.
\]
The range of relative spectral risk is $\gamma\geq 1$, and  $1\leq 1/\tilde{\sigma}_1^2\leq \gamma \leq 1/\tilde{\sigma}_r^2$. The minimum value 1 is achieved, if and only if $r=1$. The range of relative fitting error is $\epsilon>1$. If the row-space of data completely covers the target, then $\lim_{\lambda'\rightarrow 0} \epsilon=1$, but the minimum value 1 cannot be achieved unless $\lambda'=0$.
\end{proposition}

\begin{proof}
We prove the four spectral expressions, the remaining results are apparent. First, we have $\tilde{D}^*=FA^T(AA^T+\lambda'I)^{-1}=FV^T\Delta U^T$, where $\Delta$ is a diagonal matrix with $\Delta_{ii}=\sigma_i/(\sigma_i^2+\lambda')$, and $\tilde{D}^*A=FV^T\Lambda V$, where $\Lambda_{ii}=\sigma_i^2/(\sigma_i^2+\lambda')$.

$\alpha$ is easy to obtained by the expression of $\tilde{D}^*$, and then
\[
\begin{split}
\gamma&=\frac{\|\tilde{D}^*\|_F^2\|A\|_F^2}{\|\tilde{D}^*A\|_F^2}
=\frac{\Big{(}\sum_i a_i^2\sigma_i^2/(\sigma_i^2+\lambda')^2\Big{)}\Big{(}\sum_{i} \sigma_i^2\Big{)}}{\sum_i a_i^2\sigma_i^4/(\sigma_i^2+\lambda')^2}
=\frac{\Big{(}\sum_i {a'}_i^2/\sigma_i^2\Big{)}\Big{(}\sum_{i} \sigma_i^2\Big{)}}{\sum_i {a'}_i^2}
=\sum_{i}\frac{\tilde{a'}_i^2}{\tilde{\sigma}_i^2}.
\end{split}
\]
Next, $f=\|F-FV^T\Lambda V\|_F^2=\|FV^T\Lambda' V\|_F^2+\|F\bar{V}^T\bar{V}\|_F^2$, where ${\Lambda'}_{ii}=\lambda'/(\sigma_i^2+\lambda')$ and $\bar{V}$ is the complement of $V$. By $\|FV^T\Lambda' V\|_F^2=\sum_i a_i^2{\lambda'}^2/(\sigma_i^2+\lambda')^2$ and $\|F\bar{V}^T\bar{V}\|_F^2=\|F\|_F^2-\|FV^TV\|_F^2=n-\sum_i a_i^2$, $f$ is obtained. Finally, based on the expressions of $\tilde{D}^*A$ and $f$, $\epsilon$ can be obtained.
\end{proof}
Un-regularized linear regression is a special case of regularized linear regression. When setting $\lambda'=0$, all the above expressions reduce to the forms in Section~\ref{sec:definition error and risk}.

Next, it is not hard to demonstrate the correspondence relationship.
\begin{theorem}
The regularized linear regression achieves a tradeoff between fitting error and spectral risk:
\begin{equation}\label{equ:rlr abs}
g(\tilde{D}^*)=f+\lambda'\alpha,
\end{equation}
\begin{equation}\label{equ:rlr rel}
\frac{g(\tilde{D}^*)}{\|D^*A\|_F^2}=\epsilon+\lambda\gamma-1,
\end{equation}
where $\lambda'=\lambda \|A\|_F^2$.
\end{theorem}

Remember that in Section~\ref{sec:r lrc}, $\lambda'=\lambda \|A\|_F^2$ is set for the ease of parameter tuning. Here, coincidentally, it plays another role.\footnote{The correspondence is unexpected beforehand, especially (\ref{equ:rlr rel}). When we conceiving the definitions of both absolute and relative spectral risk, they are actually inspired by the structure of the solution $D^*$ of regularized LRC, rather than by the objective. Only the relative fitting error is designed after the observation of the correspondence.} The relationship (\ref{equ:rlr rel}) is more intuitive, since both $\epsilon$ and $\gamma$ have the same normalized range, and are traded off via a weight $\lambda$. The value of (\ref{equ:rlr rel}) can be compared between different data of the same model, while (\ref{equ:rlr abs}) cannot.

Finally, we rigorously study the tradeoff effect from another perspective: how the fitting error and spectral risk change when we impose regularization?
\begin{theorem}\label{theo:rlr D}
Compared with linear regression (\ref{equ:lr}), the fitting error of regularized linear regression (\ref{equ:rlr}) is increased, while the spectral risk is reduced. In precise, 1. $f(\tilde{D}^*)>f(D^*)$, 2. $\epsilon(\tilde{D}^*)>\epsilon(D^*)$, 3. $\alpha(\tilde{D}^*)<\alpha(D^*)$, 4. $\gamma(\tilde{D}^*)<\gamma(D^*)$ if $\{\sigma_i\}_{i\in \Omega}$ are not uniform ($|\Omega|>1$), where $\Omega$ is the index set of all nonzero projections $\Omega=\{j|a_j\neq 0,j=1,\dots,r\}$, $\gamma(\tilde{D}^*)=\gamma(D^*)$ otherwise.
\end{theorem}

\begin{proof}
Assertions 1 and 3 are easy to verify by comparing the spectral expressions in Proposition~\ref{pro:rlr} with those of when setting $\lambda'=0$. Assertion 2 can be established by noting $\epsilon(\tilde{D}^*)=f(\tilde{D}^*)/\|\tilde{D}^*A\|_F^2+1$, and the numerator increases while the denominator decreases. The final assertion 4 is less obvious, since both $\|\tilde{D}^*\|_F^2$ and $\|\tilde{D}^*A\|_F^2$ decrease. Some efforts have to be made.

First of all, if $\{\sigma_i\}_{i\in \Omega}$ are uniform, in particular $\Omega$ has only one member or $r=1$, it is clear that $\gamma(\tilde{D}^*)=\gamma(D^*)$. Next, we deal with the remaining case. Although ${a'}_i<a_i$ for all $i$, the ratios of decrement are different. Denote $d_i=\sigma_i^4/(\sigma_i^2+\lambda')^2$, observe that $d_1\geq d_2\geq\dots\geq d_r$ and there is at least one ``$>$'' among $\{d_i\}_{i\in\Omega}$. Focusing on $\tilde{a'}_i^2-\tilde{a}_i^2$,
\[
\tilde{a'}_i^2-\tilde{a}_i^2=\frac{a_i^2d_i}{\sum_i a_i^2d_i}-\frac{a_i^2}{\sum_i a_i^2}=a_i^2(\frac{d_i}{\sum_i a_i^2d_i}-\frac{1}{\sum_i a_i^2}).
\]
$d_i/\sum_i a_i^2d_i-1/\sum_i a_i^2$, $i=1,\dots,r$ is a descending sequence. In the first term, $d_1/(a_1^2d_1+a_2^2d_2+\dots+a_r^2d_r)=1/(a_1^2+a_2^2(d_2/d_1)+\dots+a_r^2(d_r/d_1))> 1/\sum_i a_i^2$, since $d_1\geq d_i$ for all $i>1$ and $\{d_i\}_{i\in\Omega}$ are not uniform, thus $d_1/\sum_i a_i^2d_i-1/\sum_i a_i^2> 0$. Likewise, we can prove the last term $d_r/\sum_i a_i^2d_i-1/\sum_i a_i^2<0$. Assume the first $k$ terms are positive, and the remaining terms negative, finally we have
\[
\begin{split}
\gamma(\tilde{D}^*)-\gamma(D^*)=&\sum_{i}\frac{\tilde{a'}_i^2-\tilde{a}_i^2}{\tilde{\sigma}_i^2}\\
=&\sum_{p=1}^k \frac{a_p^2}{\tilde{\sigma}_p^2}(\frac{d_p}{\sum_i a_i^2d_i}-\frac{1}{\sum_i a_i^2})\;+\;\sum_{q=k+1}^r \frac{a_q^2}{\tilde{\sigma}_q^2}(\frac{d_q}{\sum_i a_i^2d_i}-\frac{1}{\sum_i a_i^2})\\
<&\frac{1}{\tilde{\sigma}_k^2}\sum_{p=1}^k a_p^2(\frac{d_p}{\sum_i a_i^2d_i}-\frac{1}{\sum_i a_i^2})
\;+\;\frac{1}{{\tilde{\sigma}_{k+1}^2}}\sum_{q=k+1}^r a_q^2(\frac{d_q}{\sum_i a_i^2d_i}-\frac{1}{\sum_i a_i^2})\\
=&\frac{c}{\tilde{\sigma}_k^2}-\frac{c}{\tilde{\sigma}_{k+1}^2}<0,
\end{split}
\]
where $c=\sum_{p=1}^k a_p^2(\frac{d_p}{\sum_i a_i^2d_i}-\frac{1}{\sum_i a_i^2})>0$. Note that $\sum_{q=k+1}^r a_q^2(\frac{d_q}{\sum_i a_i^2d_i}-\frac{1}{\sum_i a_i^2})=-c$, since $\sum_{i}\tilde{a'}_i^2-\tilde{a}_i^2=0$.
\end{proof}

The assertion 4 has some indications. If the singular values are totally different, the nonuniform condition must hold, and the spectral risk must strictly decrease. Conversely, if the singular values are totally uniform, including the special case $r=1$, the spectral risk will not decrease. It indicates that, not in all cases, employing regularization will help to improve the generalization performance. In a balanced system, where the singular values are uniform, regularization is not necessary. Although it is an exceptional case, this point is not easily observed by the traditional Bayesian view. A practical implication of the result is that if the fitting error is under control, designing a balanced system is preferable.

We highlight the spectral risk of the uniform case in the following corollary.
\begin{corollary}
If the singular values of the data matrix are uniform, $\sigma_1=\sigma_2=\dots=\sigma_r$, then $\gamma(\tilde{D}^*)=\gamma(D^*)=r$, which goes linearly with the size of basis. In this case, the spectral risk is independent of the target vectors so long as they are not orthogonal to the data.
\end{corollary}
We encounter the spectral risk equaling to the basis size again. But the context here is different to that of Theorem~\ref{theo:nrbfn ideal risk}. This subsection's results are general, independent of the ideal graph condition.

\begin{algorithm}[tb]
   \caption{Soft KNN for finding basis (SKNN)}
   \label{alg:sknn}
\begin{algorithmic}[1]
   \STATE {\bfseries Input:} data $A\in \mathbb{R}^{p\times n}$, labels $Y\in
   \mathbb{R}^{n}$, number of neighbors $k$, confidence threshold $t$
   \STATE {\bfseries Output:} basis $G\in \mathbb{R}^{p\times r}$
   \STATE Find the $k$ nearest neighbors of each sample, record their indices, and compute their distances $Q\in \mathbb{R}^{k\times n}$,
   where $Q_{ij}$ is the Euclidian distance of the $j$th sample to its $i$th neighbor
   \STATE Set the width of Gaussian kernel $\sigma=\Sigma_{i,j} Q_{ij}/(kn)$
   \STATE Build the similarity matrix $W_{ij}=\exp\{-Q_{ij}^2/(2\sigma^2)\}$, $\forall i,j$
   \STATE Normalize the columns of $W$ so that it
   becomes probability matrix $\tilde{W}=W\diag(\mb{1}^TW)^{-1}$
   \STATE For each sample, aggregate the probability
   of its neighbors that have the same label as it, resulting in a
   confidence vector $T\in \mathbb{R}^{n}$
   \STATE Choose the samples whose $T_j<t$ as the basis $G$
   \STATE If there is any class missed, add the sample with the
   lowest confidence in this class to the basis
\end{algorithmic}
\end{algorithm}

\begin{algorithm}[tb]
   \caption{The training of nRBFN}
   \label{alg:train srbf}
\begin{algorithmic}[1]
   \STATE {\bfseries Input:} training data $A\in \mathbb{R}^{p\times n}$, labels $Y\in \mathbb{R}^{n}$, regularization weight $\lambda$, number of neighbors $k$,
   confidence threshold $t$
   \STATE {\bfseries Output:} basis $G\in \mathbb{R}^{p\times r}$, width of Gaussian kernel
   $\sigma$, weight matrix $X\in \mathbb{R}^{K\times r}$
   \STATE Find the basis via soft KNN $G=$SKNN$(A,Y,k,t)$
   \STATE Compute pairwise distance $Q_{ij}=\|G_i-A_j\|_2$, $\forall i,j$
   \STATE Set the width of Gaussian kernel $\sigma=\Sigma_{i,j} Q_{ij}/(rn)$
   \STATE Build the kernel matrix $W_{ij}=\exp\{-Q_{ij}^2/(2\sigma^2)\}$, $\forall i,j$
   \STATE Normalize the columns of the kernel matrix
   $\tilde{W}=W\diag(\mb{1}^TW)^{-1}$
   \STATE Convert the label vector $Y$ to indicator matrix $F$
   \STATE Compute the weight matrix $X=F\tilde{W}^T(\tilde{W}\tilde{W}^T+\lambda\|\tilde{W}\|_F^2I)^{-1}$
\end{algorithmic}
\end{algorithm}

\begin{algorithm}[tb]
   \caption{The testing of nRBFN}
   \label{alg:test srbf}
\begin{algorithmic}[1]
   \STATE {\bfseries Input:} test data $B\in \mathbb{R}^{p\times m}$, basis $G\in \mathbb{R}^{p\times r}$, width of Gaussian kernel
   $\sigma$, weight matrix $X\in \mathbb{R}^{K\times r}$
   \STATE {\bfseries Output:} predicted labels $l\in\mathbb{R}^{m}$
   \STATE Build the kernel matrix $(W_B)_{ij}=\exp\{-\|G_i-B_j\|_2^2/(2\sigma^2)\}$, $\forall i,j$
   \STATE Normalize the columns of the kernel matrix
   $\tilde{W}_B=W_B\diag(\mb{1}^TW_B)^{-1}$
   \STATE Compute the predicted indicator matrix $\hat{F}=X\tilde{W}_B$
   \STATE Obtain the predicted labels
   $l_j=\mathop{\arg\max}_k\,\hat{F}_{kj}$, $\forall j$
\end{algorithmic}
\end{algorithm}

\section{Basis Selection Strategy}\label{sec:basis selection}
To make nRBFN work, we have to settle the basis selection problem and associated parameter setting. Unfortunately, it is still an open question to find the optimal basis for (n)RBFN. In this paper, we are contented with a strategy that is easy to use yet can deliver good performance.

Traditionally, there are three problems (n)RBFN needs to deal with: (1) how to decide the basis size? (2) how to select the basis? (3) how to set the Gaussian width? Usually, the size of the basis
is manually assigned, except for some incremental learning methods \cite{Chen1991Orthogonal,Bugmann1998Normalized,Huang2004An,Yu2014An} which learn the basis vectors one by one. There are two kinds of methods for basis selection: one is gradient descent method \cite{Poggio1990Networks,Karayiannis1999Reformulated,Xie2012Fast}, the other is sample-selection based method, including: (1) random selection of a subset of the training set, (2) clustering the data and using the cluster centers as the basis \cite{Moody1989Fast,Bishop1991Improving,Musavi1992On,Que2016Back}, and (3) incremental learning methods.  Generally, the gradient descent methods are time-consuming and lose one of the main advantages of RBFN compared with traditional neural networks. The incremental learning methods are also expensive and complex. The most frequently applied basis selection strategy is the clustering based method, especially K-means, due to its efficiency. The Gaussian width can be set via
some heuristics \cite{Moody1989Fast,Schwenker2001Three,Du2014Neural}, e.g., the maximum distance
between basis vectors, or learnt by gradient descent, or searched via model
selection. It was reported that the performance is not so sensitive to this parameter \cite{Specht1990Probabilistic}, especially for nRBFN \cite{Bugmann1998Normalized}.

In this paper, we devise a strategy that chooses the samples near the
boundaries of different classes as the basis. The basis size is determined by a confidence parameter that is
much easier to set. The Gaussian width is set as the
mean distance of the training set to the basis.

RBFN has its origin in function
approximation. The basis is regarded as templates or stereotypical
patterns. It is this view that leads to the clustering heuristics \cite{Scholkopf1997Comparing}.
However, the classification problem is different from the general
regression problem that is not concerned with the separability of
classes. For classification, samples near the boundaries may deliver
more crucial information for the separation of classes than those
in the inner part. This has been investigated by  \cite{Scholkopf1997Comparing,Bugmann1998Normalized,Oyang2005Data}. Usually, the idea is borrowed from SVM \cite{Cortes1995Support}, where the boundary points are called support vectors.

The boundary points can be identified by their classification
confidence. That is, for each sample, if the probability belonging
to its labeled class is known, then a sample can be identified as
boundary point if this probability is below some preassigned
threshold. Soft KNN may be the simplest tool that meets this demand.
The detailed algorithm is presented in Algorithm~\ref{alg:sknn}.
Note the basis size is implicitly determined by the
confidence threshold (within range (0,1]), which is easy
to set due to clear interpretation. In this way, on the one hand, the basis size can be determined according to the complexity of data distribution: when the classes overlap more, more samples are recruited as basis, vice versa. On the other hand, we can flexibly control the tradeoff between accuracy and resource burden: larger $t$ implies better accuracy but higher computational cost, vice versa.

Finally, the training and testing algorithms of nRBFN are shown in
Algorithm~\ref{alg:train srbf} and Algorithm~\ref{alg:test srbf} respectively. The time complexity of SKNN is $O(n^2(p+\log n))$, dominated by the computation of Euclidean distance and search of nearest neighbors. The time complexity of training nRBFN excluding SKNN is $O(prn+r^2n)$, and the total is $O(n^2(p+\log n)+r^2n)$, generally dominated by SKNN. Lastly, the complexity of testing is $O((p+K)rm)$.

\section{Experiments}\label{sec:experiment}
The experiments consist of two parts: demonstrating the performance of nRBFN, and evaluating the error and risk of nRBFN and LRC.

The experiments are carried out on a set of benchmark data sets, shown in Table~\ref{tab:srbf vs others}.\footnote{The data sets come from \url{http://www.csie.ntu.edu.tw/~cjlin/libsvmtools/datasets/} and \url{http://www.cad.zju.edu.cn/home/dengcai/Data/data.html}.} The data sets include classical small data sets of UCI Machine Learning Repository \cite{Lichman:2013}: iris, wdbc, glass, sonar, wine; high-dimensional and small-sample-size gene data: colon, leukemia; human face images: ORL, AR, YaleB; high-dimensional and large-sample-size text data: TDT2, 20news; and large-sample-size hand-written digit images: USPS, MNIST. If the original data set does not have a training-testing split, we use the first half of each class as the training set, except glass, sonar, and YaleB, where random splits have been performed to avoid particular sample sequences. Following common practice, each face image of ORL, AR, and YaleB is normalized to unit length. The procedures are run on a server with 32GB memory and 24 cores CPU of 2.93GHz.\footnote{The high-performance machine is used only for the purpose of convenience rather than necessity for the experiments.} In the results below, test error refers to the classification error on test set.

\subsection{Performance of nRBFN}
We first evaluate the parameters of nRBFN, then compare the performance of nRBFN with some other algorithms.
\subsubsection{Evaluation of the Parameters}\label{sec:para}
We evaluate the influence of the three parameters of nRBFN, $\lambda$, $t$, and $k$, on the classification performance. The results on five representative data sets are shown in Figure~{\ref{fig:para}}. Three default values $\lambda=10^{-13}$, $t=0.9$, and $k=20$ are used. When one parameter varies, the others are fixed with the default values. We find that:

(1) The test error generally decreases as $\lambda$ decreases, and reaches a plateau after $10^{-8}$. The exceptional case is iris, on which, due to insufficient sampling, larger $\lambda$ is needed to avoid noisy subspace. Nevertheless, the difference is not large. Considering methods using regularization are usually plagued by the problem of tuning the regularization weight, nRBFN shows a desirable feature: as a rule of thumb, $\lambda<10^{-8}$ generally delivers near-optimal result of nRBFN. Our experience showed that it also holds for the other data sets we have tested. This rule will be justified further from the error-and-risk perspective in Section~\ref{sec:risk lambda t}. Remember the actual regularization weight is $\lambda'=\lambda \|\tilde{W}\|_F^2$. It is adaptive to the data. The principles of this setting are provided by the row-space projection view and error-and-risk analysis. If $\lambda'$ is set as a whole, the above rule no long holds, and we have to search for the optimal value in a wide range.

(2) The error steadily decreases as $t$ increases, as expected. Note, in essence, $t$ is not merely a parameter to be tuned, it is also a choice of ours. It virtually controls the tradeoff between accuracy and resource burden. Our experience suggests $t=0.9$ strikes a good balance for general data sets. It will be justified further in below experiments.

(3) The performance is not sensitive to $k$. It becomes stable after $k\geq 8$.

The experiments show that the parameter setting of nRBFN is easy. $t$ is a matter of choice. $k$ has minor impact on the performance. There is only one parameter, $\lambda$, needed to be tuned, but its determination is not difficult. In the following, we fix $t=0.9$, $k=20$, and for each data set, $\lambda$ will be selected via 5-fold cross validation over three values $\{10^{-5},10^{-9},10^{-13}\}$.

\begin{figure*}[htb]
\center{

\subfigure[regularization weight $\lambda$]{\label{fig:para:lambda} 
\includegraphics[width=4.7cm]{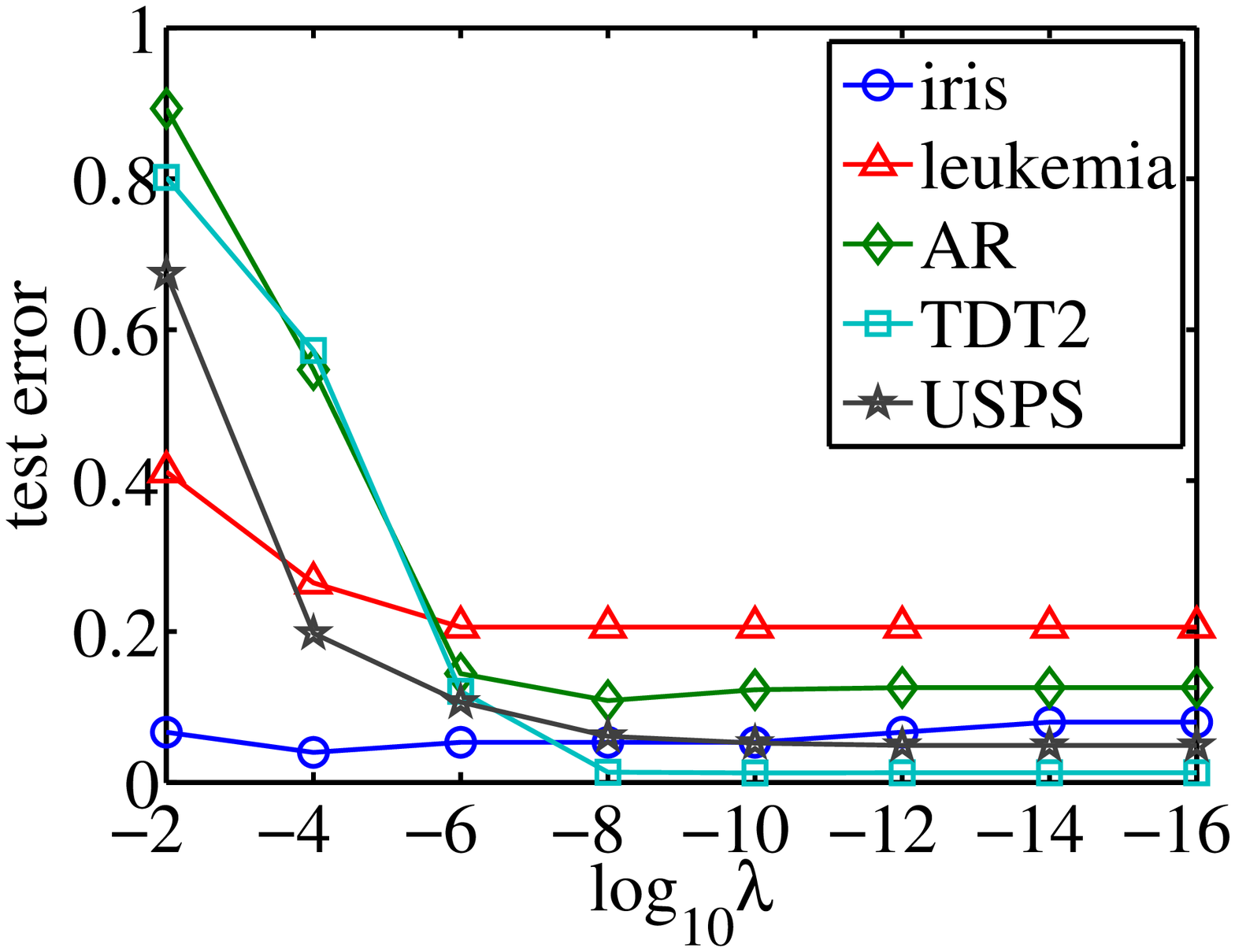}}
\subfigure[confidence threshold $t$]{\label{fig:para:t} 
\includegraphics[width=4.5cm]{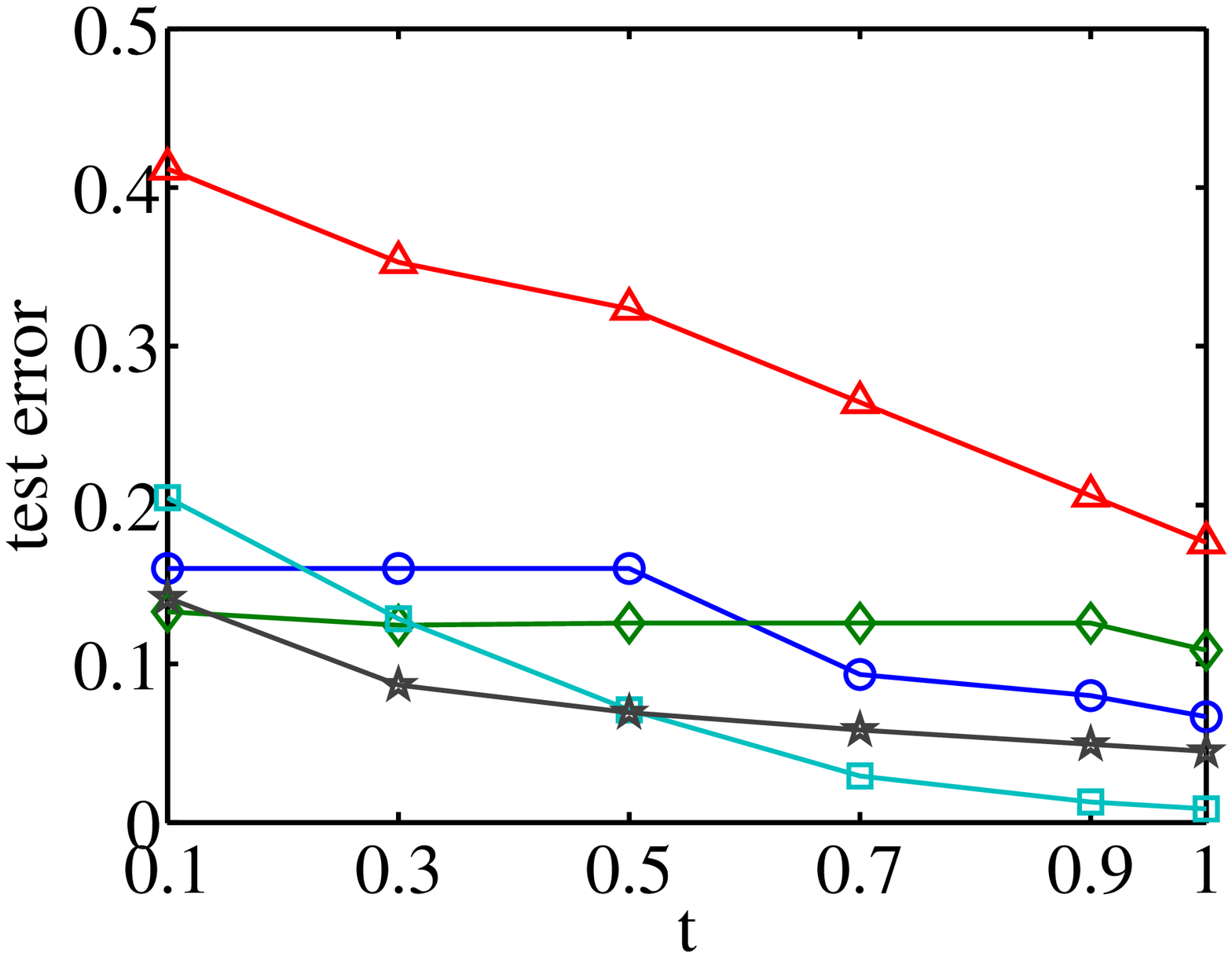}}
\subfigure[number of neighbors $k$]{\label{fig:para:k} 
\includegraphics[width=4.5cm]{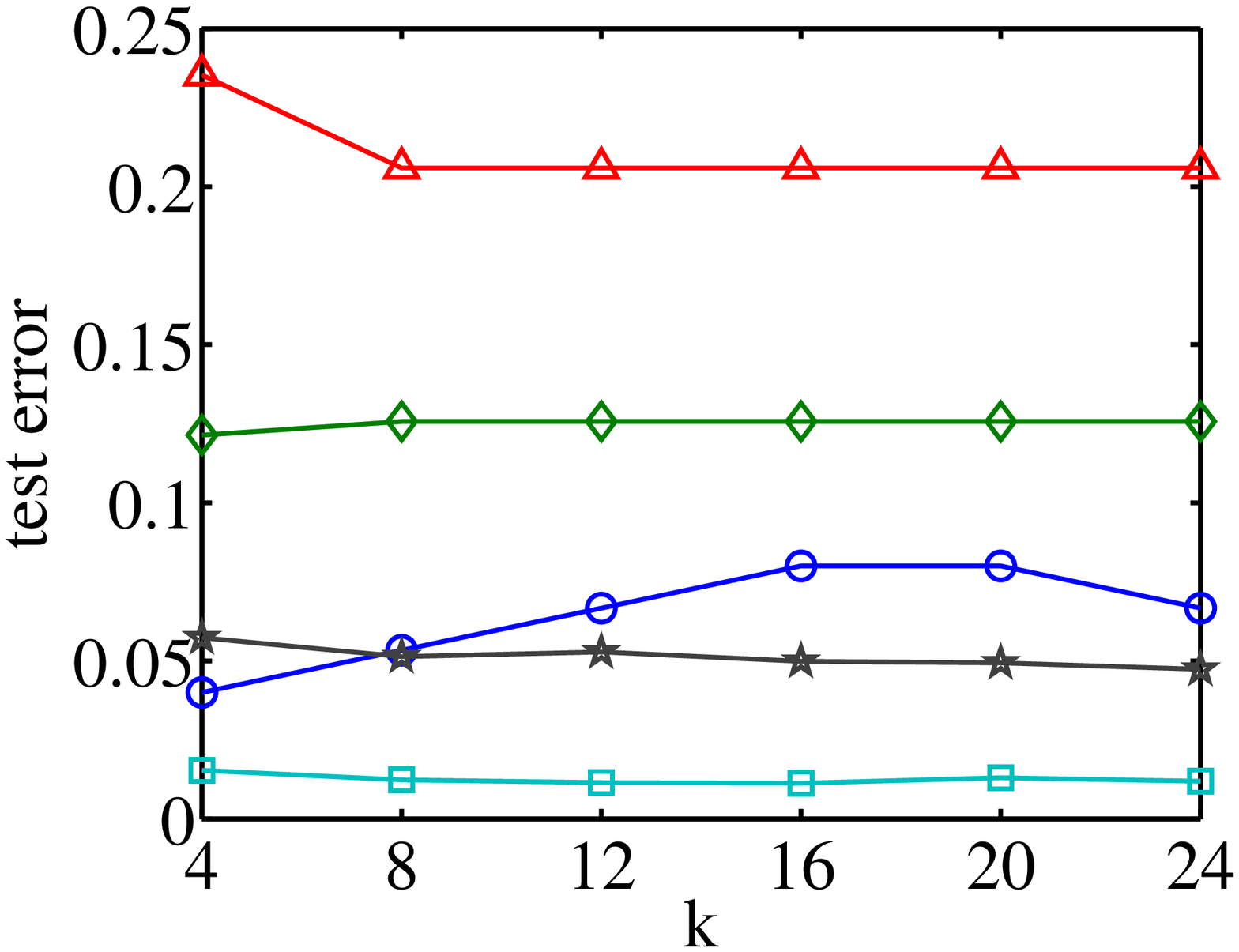}} }

\caption{The influence of the three parameters of nRBFN, $\lambda$, $t$, and $k$, on the classification performance.}\label{fig:para}
\end{figure*}

\subsubsection{nRBFN v.s. Other Classification Algorithms}
In this subsection, we compare the classification performance of nRBFN with some other
algorithms. The results are shown in Table~\ref{tab:srbf vs others}. The involved algorithms include: (1) KNN (k=20), (2) LRC ($\lambda$ is
selected via 5-fold cross validation over
$10^{\{-13,-12,\dots,-2\}}$), (3) ROLS (regularized orthogonal least squares algorithm for RBFN) \cite{Chen1996Regularized} (a classical incremental learning method for basis selection, $\lambda$ is
selected using the same scheme as nRBFN, $\sigma$ and basis size are provided by nRBFN), (4) RBFNnl (RBFN from Netlab toolbox) \cite{Nabney2003Algorithms} (a traditional RBFN that
finds the basis via clustering--gaussian mixture model, without regularization, the width of
Gaussian is set as the maximum distance between the basis vectors,
bias parameters are included, basis size follows nRBFN), (5) nRNwr (nRBFN without regularization, using the same
basis as nRBFN), (6) nRNrb (nRBFN with basis chosen randomly from the
training set, basis size follows nRBFN), (7) SVM \cite{Chang2007LIBSVM} (Gaussian
kernel, the weight $C$ is selected over $2^{\{-1,0,\dots,12\}}$ and $\sigma$ is
selected over $1.4^{\{-4,-3,\dots,4\}}\times \sigma$ of nRBFN via
5-fold cross validation). Except RBFNnl and SVM, the others are
implemented by us using MATLAB.

\begin{table*}[htp]
\caption{Comparing test error (\%) of nRBFN with some algorithms on benchmark data sets. ``size'' denotes data dimension
$\times$ (\#training samples + \#test samples).}\label{tab:srbf vs
others}
\begin{center}
\begin{scriptsize} %
\begin{tabular}{|c|c|c|*{8}{|c}|}
\hline data & size & \#classes & KNN & LRC & ROLS & RBFNnl & nRNwr &
nRNrb & SVM & nRBFN\\\hline\hline


iris & 4$\times$(75+75) & 3 & 5.3 & 18.7 & \tb{4.0} & 5.3 & 6.7 & 5.3 & 5.3 & 5.3\\\hline

wdbc & 30$\times$(285+284) & 2 & 6.3 & 6.3 & 8.5 & 4.9 & 7.7 & 5.6 &
\tb{4.6} & 4.9\\\hline

glass & 9$\times$(109+105) & 6 & 39.0 & 43.8 & 36.2 & 67.6 & 50.5 & 36.2 &
\tb{31.4} & 38.1\\\hline

sonar & 60$\times$(105+103) & 2 & 39.8 & 23.3 & \tb{18.4} & 22.3 & 21.4 &
\tb{18.4} & \tb{18.4} & \tb{18.4}\\\hline

wine & 13$\times$(90+88) & 3 & 33.0 & 3.4 & \tb{1.1} & 5.7 & 6.8 & \tb{1.1} &
2.3 & \tb{1.1}\\\hline

colon & 2000$\times$(31+31) & 2 & 35.5 & \tb{16.1} & \tb{16.1} & \tb{16.1} &
\tb{16.1} & \tb{16.1} & \tb{16.1} & \tb{16.1}\\\hline

leukemia & 7129$\times$(38+34) & 2 & 41.2 & \tb{17.6} & 20.6 & \tb{17.6} &
20.6 & 20.6 & 23.5 & 20.6\\\hline

ORL & 1024$\times$(200+200) & 40 & 54.0 & \tb{7.5} & 8.5 & 10.0 & 8.5 &
8.5 & 10.0 & 8.5\\\hline

AR & 1024$\times$(700+700) & 100 & 69.4 & \tb{6.9} & 11.3 & 9.9 & 12.6 &
11.3 & 26.0 & 11.3\\\hline

YaleB & 1024$\times$(1209+1205) & 38 & 46.1 & 1.6 & \tb{1.2} & 1.5 & \tb{1.2} &
\tb{1.2} & 3.2 & \tb{1.2}\\\hline

TDT2 & 36771$\times$(4703+4691) & 30 & 1.2 & - & - & \tb{0.9} & 98.5 &
1.0 & 1.7 & 1.3\\\hline

20news & 26214$\times$(11314+7532) & 20 & 24.0 & - & - & 15.0 & 77.5 &
14.6 & 15.1 & \tb{14.3}\\\hline

USPS & 256$\times$(7291+2007) & 10 & 8.2 & 12.6 & - & 5.0 & 5.0 & 5.2 &
\tb{4.6} & 4.9\\\hline

MNIST & 784$\times$(60000+10000) & 10 & 3.8 & 13.4 & - & 2.0 & 1.9 & 2.0
& \tb{1.5} & 1.8\\\hline

\end{tabular}
\end{scriptsize}
\end{center}
\end{table*}

The best scores spread across the table. It is hard for a method to dominate over all others on such data sets of diverse nature. However, when we compare them pair-wise, the advantage of nRBFN becomes prominent.

(1) As a linear model, the performance of LRC is limited, as the ideal graph condition implies. It mainly preforms well on high-dimensional data where $p\geq n$, since in this case the rank of row space is full. Results on TDT2 and 20news are absent, since it fails to run on such big data.

(2) The results of nRNwr are inferior to nRBFN, confirming the importance of regularization.

(3) nRBFN generally outperforms ROLS, RBFNnl, and nRNrb, showing the effectiveness of the basis selection strategy of nRBFN. ROLS is resource-consuming, it is unable to run on big data sets. Note that it is hard for RBFNnl and nRNrb themselves to determine a suitable basis size. The strategy of nRBFN, which implicitly controls the size by a user-friendly confidence threshold, makes it much easier.

(4) nRBFN generally obtains better results than SVM on these data sets. Since performance depends on the basis sizes, we list them in Table~\ref{tab:basis size}. The support vectors of SVM are found automatically, we note on the first ten smaller data sets, the bases found by nRBFN when setting $t=0.9$ have sizes roughly consistent with those of SVM. It means the basis size has been determined properly according to the complexity of data distribution: when the classes overlap more, more points will be selected as basis, and vice versa. When setting $t=0.9$, most of the uncertain points helpful for determining the classification boundaries have been included. For human face images, it is well-known that the data are clustered according to lightening, expressions, poses, rather than identity. For example, a left lightening cluster may include images from all identities. Thus, almost all points serve as basis. On the four larger sets, the bases obtained by nRBFN are more economical than those of SVM, however, the test error are not necessarily worse.

\begin{table}[htp]
\begin{minipage}[b]{0.48\linewidth}
\caption{A comparison of the basis sizes of SVM and nRBFN ($t=0.9$). $n$ is the size of training set. ``basis (\%)'' denotes the proportion of the basis vectors with respect to the training set.}\label{tab:basis size}
\vskip 0.15in
\vskip -0.1in
\begin{center}
\vskip -0.1in
\begin{scriptsize} %
\begin{tabular}{|c||c||C{0.5cm}|C{0.8cm}||C{0.5cm}|C{0.8cm}|}
\hline \multirow{2}{*}{data} & \multirow{2}{*}{$n$}

& \multicolumn{2}{c||}{basis (\%)} & \multicolumn{2}{c|}{test error (\%)} \\\cline{3-6}

& & SVM & nRBFN & SVM & nRBFN \\\hline\hline

iris & 75 & \tb{36.0} & 42.7 & 5.3 & 5.3\\\hline

wdbc & 285 & \tb{23.2} & 25.6 & \tb{4.6} & 4.9\\\hline

glass & 109 & \tb{74.3} & 92.7 & \tb{31.4} & 38.1\\\hline

sonar & 105 & \tb{59.0} & 100.0 & 18.4 & 18.4\\\hline

wine & 90 & \tb{55.6} & 82.2 & 2.3 & \tb{1.1}\\\hline

colon & 31 & \tb{90.3} & 100.0 & 16.1 & 16.1\\\hline

leukemia & 38 & \tb{86.8} & 100.0 & 23.5 & \tb{20.6}\\\hline

ORL & 200 & \tb{99.5} & 100.0 & 10.0 & \tb{8.5}\\\hline

AR & 700 & \tb{99.1} & 100.0 & 26.0 & \tb{11.3}\\\hline

YaleB & 1209 & \tb{95.3} & 100.0 & 3.2 & \tb{1.2}\\\hline

TDT2 & 4703 & 70.3 & \tb{27.8} & 1.7 & \tb{1.3}\\\hline

20news & 11314 & 79.5 & \tb{71.7} & 15.1 & \tb{14.3}\\\hline

USPS & 7291 & 32.5 & \tb{19.9} & \tb{4.6} & 4.9\\\hline

MNIST & 60000 & 30.4 & \tb{14.6} & \tb{1.5} & 1.8\\\hline
\end{tabular}
\end{scriptsize}
\end{center}
\vskip -0.1in

\end{minipage}%
  \hspace{2mm}
\begin{minipage}[b]{0.47\linewidth}
  \caption{nRBFN v.s. SVM. nRBFN has the same basis size as SVM. ``basis (\%)'' denotes the proportion of support vectors. The time cost includes the training and testing
time.}\label{tab:srbf vs svm}
\vskip -0.1in
\begin{center}
\vskip -0.1in
\begin{scriptsize} %
\begin{tabular}{|c||c||C{0.5cm}|C{0.8cm}||C{0.7cm}|C{0.8cm}|}
\hline \multirow{2}{*}{data} & \multirow{2}{*}{basis(\%)}

& \multicolumn{2}{c||}{test error (\%)} & \multicolumn{2}{c|}{time
cost (s)} \\\cline{3-6}

& & SVM & nRBFN & SVM & nRBFN \\\hline\hline

iris  & 36.0 & 5.3 & 5.3 & - & -\\\hline
wdbc  & 23.2 & \tb{4.6} & 4.9 & - & -\\\hline
glass  & 74.3 & \tb{31.4} & 37.1 & - & -\\\hline
sonar  & 59.0 & 18.4 & 18.4 & - & -\\\hline
wine  & 55.6 & 2.3 & \tb{1.1} & - & -\\\hline
colon  & 90.3 & 16.1 & 16.1 & \tb{$<$0.1} & 0.1\\\hline
leukemia  & 86.8 & 23.5 & \tb{17.6} & 0.2 & \tb{0.1}\\\hline
ORL  & 99.5 & 10.0 & \tb{8.0} & 0.6 & \tb{0.1}\\\hline
AR  & 99.1 & 26.0 & \tb{11.1} & 6.7 & \tb{0.8}\\\hline
YaleB  & 95.3 & 3.2 & \tb{1.4} & 19.5 & \tb{1.4}\\\hline
TDT2  & 70.3 & 1.7 & \tb{0.9} & 77.2 & \tb{16.5}\\\hline
20news  & 79.5 & 15.1 & \tb{14.3} & 263.1 & \tb{82.3}\\\hline
USPS  & 32.5 & \tb{4.6} & 4.7 & \tb{18.4} & 24.0\\\hline
MNIST  & 30.4 & \tb{1.5} & 1.6 & \tb{1289.1} & 2455.3\\\hline

\end{tabular}
\end{scriptsize}
\end{center}
\vskip -0.1in
\end{minipage}

\vskip -0.2in
\end{table}

\subsubsection{nRBFN v.s. SVM with the Same Basis Size}
Next, for a fair comparison between nRBFN and SVM, we let the basis size of nRBFN to be equal to the size of support vectors. This is done by choosing the specific number of samples with the lowest confidence as the basis. The other parameters are set as before. The results are shown in Table~\ref{tab:srbf vs svm}. In this test, nRBFN performs even better, and the time cost is comparable to that of SVM. Note that, the time cost does not include the part of cross validation. SVM usually needs to run hundreds of times during cross validation, while nRBFN only runs a dozen times.

\subsubsection{nRBFN with fixed parameters v.s. SVM}
Finally, we test the performance of nRBFN with a fixed set of parameters ($\lambda=10^{-13}$, $t=0.9$, $k=20$). In this case, nRBFN involves no model selection, but the results are still not far from those of SVM, as Table~\ref{tab:srbf fix lambda vs svm} shows.


\subsection{Empirical Evaluation of Error and Risk}
In this section, we empirically evaluate the fitting error and spectral risk of nRBFN and LRC, and investigate their influence on the performance.

\begin{table}[htp]
\begin{minipage}[b]{0.275\linewidth}
\caption{SVM v.s. nRBFN with fixed parameters ($\lambda=10^{-13}$, $t=0.9$, $k=20$).}\label{tab:srbf fix lambda vs svm}
\vskip -0.1in
\begin{center}
\vskip -0.1in
\begin{scriptsize} %
\begin{tabular}{|c||c|c|}
\hline
data & SVM & nRBFN \\\hline\hline
iris & \tb{5.3} & 8.0 \\\hline
wdbc & \tb{4.6} & 5.3 \\\hline
glass & \tb{31.4} & 35.2 \\\hline
sonar & \tb{18.4} & 21.4 \\\hline
wine & 2.3 & \tb{1.1} \\\hline
colon & 16.1 & 16.1 \\\hline
leukemia & 23.5 & \tb{20.6} \\\hline
ORL & 10.0 & \tb{8.5} \\\hline
AR & 26.0 & \tb{12.6} \\\hline
YaleB & 3.2 & \tb{1.2} \\\hline
TDT2 & 1.7 & \tb{1.3} \\\hline
20news & \tb{15.1} & 15.2 \\\hline
USPS & \tb{4.6} & 4.9 \\\hline
MNIST & \tb{1.5} & 1.8 \\\hline
\end{tabular}
\end{scriptsize}
\end{center}
\vskip -0.1in

\end{minipage}%
  \hspace{2mm}
\begin{minipage}[b]{0.55\linewidth}
\caption{The effect of regularization in terms of error-and-risk: nRNwr (nRBFN without regularization) v.s. nRBFN. The data sets on which significant changes occur are marked with boldface.}\label{tab:srbfwr vs srbf}
\vskip 0.1in
\vskip -0.1in
\begin{center}
\vskip -0.1in
\begin{scriptsize} %
\begin{tabular}{|c||c||c|c||c|c||c|c|}
\hline \multirow{2}{*}{data} & \multirow{2}{*}{$\lambda$}

& \multicolumn{2}{c||}{fitting error} & \multicolumn{2}{c||}{spectral risk} & \multicolumn{2}{c|}{test error (\%)}\\\cline{3-8}

&  & nRNwr & nRBFN & nRNwr & nRBFN & nRNwr & nRBFN\\\hline\hline

\tb{iris} & 1e-09 & 1.022 & 1.044 & 9.7e+09 & 2.5e+06 & 6.7 & 5.3\\\hline
\tb{wdbc} & 1e-09 & 1.045 & 1.106 & 2.6e+15 & 2.4e+06 & 7.7 & 4.9\\\hline
\tb{sonar} & 1e-05 & 1.000 & 1.191 & 1.5e+06 & 5.5e+03 & 21.4 & 18.4\\\hline
colon & 1e-05 & 1.000 & 1.001 & 7.2e+02 & 6.0e+02 & 16.1 & 16.1\\\hline
leukemia & 1e-05 & 1.000 & 1.001 & 1.2e+03 & 1.0e+03 & 20.6 & 20.6\\\hline
ORL & 1e-09 & 1.000 & 1.000 & 5.5e+05 & 5.4e+05 & 8.5 & 8.5\\\hline
\tb{AR} & 1e-09 & 1.000 & 1.014 & 9.2e+07 & 2.5e+07 & 12.6 & 11.3\\\hline
YaleB & 1e-13 & 1.000 & 1.000 & 1.4e+09 & 8.2e+08 & 1.2 & 1.2\\\hline
USPS & 1e-13 & 1.056 & 1.056 & 4.1e+08 & 4.1e+08 & 5.0 & 4.9\\\hline
MNIST & 1e-13 & 1.064 & 1.064 & 2.4e+10 & 1.6e+10 & 1.9 & 1.8\\\hline
\end{tabular}
\end{scriptsize}
\end{center}
\vskip -0.1in
\end{minipage}

\vskip -0.2in
\end{table}

\subsubsection{The Effect of Regularization: nRNwr v.s. nRBFN}
We evaluate the effect of regularization through comparing nRNwr (nRBFN without regularization) and nRBFN. The results are shown in Table~\ref{tab:srbfwr vs srbf}. Due to rank deficiency ($\rank(\tilde{W})<r$), results of nRNwr on glass, wine, TDT2, and 20news are erroneous and not shown. According to the table, on the sets marked with bold face, regularization significantly reduces the spectral risk, mostly by orders of magnitude, while the fitting error is increased slightly, but not yet sacrificed much. On all these sets, the test error have been reduced. This demonstrates the effectiveness of regularization. On the other data sets, regularization affects not much. Since $\lambda$ has been selected for optimal performance, it implies that on these data sets the found row-subspaces by nRNwr probably contain discriminative features rather than noise. Note that, most of these data sets have large sample sizes, indicating sufficient sampling. Even though, regularization does not provide better result in this case, it does not undermine the performance--so long as $\lambda$ is set properly. Considering its capability in dealing with both insufficient and sufficient sampling cases, regularization should be applied.

\subsubsection{LRC v.s. nRBFN}
We compare the fitting error ($\epsilon$), spectral risk ($\gamma$), the tradeoff ($\epsilon+\lambda\gamma$), and the test error between LRC and nRBFN, and study their influence on the performance. The results are shown in Table~\ref{tab:lrc vs srbf}.

Theoretically, two factors act together contributing to the performance of nRBFN/LRC. One is the rank of basis/data row-space. Despite of the idea graph condition, larger rank implies better fitting capacity. For nRBFN, this can be controlled via $t$. In the extreme case of full rank, zero fitting error can be achieved (although spectral risk is not guaranteed). The other factor is the ideal graph condition. Despite that the rank may be low, so long as the condition is nearly met, both the fitting error and spectral risk will be low.

We analyze Table~\ref{tab:lrc vs srbf} by two parts. For the data sets from colon to AR, LRC performs equally or better than nRBFN. These data sets are distinct from the others in that the ranks of LRC and nRBFN are exactly the same, and are full. Since both the fitting error and spectral risk depend on the rank, the comparison between LRC and nRBFN is thus fair. The lower test error of LRC can be attributed to the lower spectral risk. Although nRBFN has lower fitting error, it has higher risk of overfitting. Note that, these data sets consist of the gene data and face images. For nRBFN, on these sets all samples are recruited as basis, indicating the classes overlap heavily. Consequently the ideal graph condition will not be met well and the spectral risk can be high. In addition, although the fitting error of LRC seem higher, we found that on these four data sets (plus YaleB) the classification error of both LRC and nRBFN are uniformly zero. All these support the lower test error of LRC.

For the data sets marked with boldface, nRBFN performs better. This can be attributed to the higher ranks and lower fitting error. As to the higher spectral risk, first, the spectral risk is proportionally related to the rank. Second, higher value of spectral risk does not mean error actually happened, but a warning is signalled. The found subspace may be noise or discriminative features. Third, the cross validation determined a much smaller weight to offset the spectral risk, so that the product $\lambda\gamma$ is generally below the magnitude of $10^{-2}$ for nRBFN ($10^{-1}$ for LRC). In contrast to the fitting error that is above 1, the weighted risk is quite small, and the tradeoff $\epsilon+\lambda\gamma\approx \epsilon$. On the one hand, it implies that the performance is overwhelmed by the fitting error. On the other hand, the large offset of nRBFN suggests that the risk warning has been ignored. The cross validation learnt that, after regularization, the warning does no harm, the found subspace was judged to be discriminative features.
These analyses also apply to the previous four data sets, since the margins of test error are not large. We conclude nRBFN performs better overall. This is due to the consistently lower fitting error and not severe overfitting.

\begin{table*}[htp]
\caption{Comparing the fitting error ($\epsilon$), spectral risk ($\gamma$), the tradeoff ($\epsilon+\lambda\gamma$), and the test error between LRC and nRBFN. ``rank'' denotes the rank of data row-space/basis, divided by $n$. Smaller values are marked with boldface.}\label{tab:lrc vs srbf}
\begin{center}
\begin{scriptsize} %
\begin{tabular}{|c||C{0.6cm}|C{0.8cm}||C{0.5cm}|C{0.8cm}||c|C{0.8cm}||C{0.6cm}|C{0.8cm}||C{0.5cm}|C{0.8cm}||C{0.5cm}|C{0.8cm}|}
\hline \multirow{2}{*}{data}

& \multicolumn{2}{c||}{rank (\%)} & \multicolumn{2}{c||}{fitting error ($\epsilon$)} & \multicolumn{2}{c||}{$\lambda$} & \multicolumn{2}{c||}{spectral risk ($\gamma$)}& \multicolumn{2}{c||}{$\epsilon+\lambda\gamma$} & \multicolumn{2}{c|}{test error (\%)}\\\cline{2-13}

& LRC & nRBFN & LRC & nRBFN & LRC & nRBFN & LRC & nRBFN & LRC & nRBFN & LRC & nRBFN\\\hline\hline

\tb{iris} & \tb{6.7} & 42.7 & 1.39 & \tb{1.04} & 1e-04 & \tb{1e-09} & \tb{9e+00} & 3e+06 & 1.39 & \tb{1.05} & 18.7 & \tb{5.3}\\\hline
\tb{wdbc} & \tb{10.9} & 25.6 & 3.38 & \tb{1.11} & 1e-04 & \tb{1e-09} & \tb{1e+03} & 2e+06 & 3.49 & \tb{1.11} & 6.3 & \tb{4.9}\\\hline
\tb{glass} & \tb{28.4} & 92.7 & 2.09 & \tb{1.24} & 1e-04 & \tb{1e-09} & \tb{1e+02} & 2e+07 & 2.10 & \tb{1.26} & 43.8 & \tb{38.1}\\\hline
\tb{sonar} & \tb{29.5} & 100.0 & 1.29 & \tb{1.19} & 1e-03 & \tb{1e-05} & \tb{5e+01} & 5e+03 & 1.33 & \tb{1.25} & 23.3 & \tb{18.4}\\\hline
\tb{wine} & \tb{34.4} & 82.2 & 1.10 & \tb{1.07} & 1e-08 & \tb{1e-13} & \tb{3e+05} & 1e+11 & 1.10 & \tb{1.08} & 3.4 & \tb{1.1}\\\hline
colon & 100.0 & 100.0 & 1.46 & \tb{1.00} & 1e-03 & \tb{1e-05} & \tb{3e+02} & 6e+02 & 1.72 & \tb{1.01} & 16.1 & 16.1\\\hline
leukemia & 100.0 & 100.0 & 1.16 & \tb{1.00} & 1e-04 & \tb{1e-05} & 2e+03 & \tb{1e+03} & 1.40 & \tb{1.01} & \tb{17.6} & 20.6\\\hline
ORL & 100.0 & 100.0 & 1.44 & \tb{1.00} & 1e-04 & \tb{1e-09} & \tb{3e+03} & 5e+05 & 1.76 & \tb{1.00} & \tb{7.5} & 8.5\\\hline
AR & 100.0 & 100.0 & 1.42 & \tb{1.01} & 1e-05 & \tb{1e-09} & \tb{2e+04} & 2e+07 & 1.63 & \tb{1.04} & \tb{6.9} & 11.3\\\hline
\tb{YaleB} & \tb{84.8} & 100.0 & 1.14 & \tb{1.00} & 1e-05 & \tb{1e-13} & \tb{1e+04} & 8e+08 & 1.25 & \tb{1.00} & 1.6 & \tb{1.2}\\\hline
\tb{USPS} & \tb{3.5} & 19.9 & 1.42 & \tb{1.06} & 1e-04 & \tb{1e-13} & \tb{1e+02} & 4e+08 & 1.43 & \tb{1.06} & 12.6 & \tb{4.9}\\\hline
\tb{MNIST} & \tb{1.3} & 14.6 & 2.15 & \tb{1.06} & 1e-03 & \tb{1e-13} & \tb{6e+01} & 2e+10 & 2.21 & \tb{1.07} & 13.4 & \tb{1.8}\\\hline

\end{tabular}
\end{scriptsize}
\end{center}
\end{table*}

\subsubsection{How Error and Risk of nRBFN Change as $\lambda$ and $t$ Vary}\label{sec:risk lambda t}
The setting is the same as Section~\ref{sec:para}, now we focus on the fitting error and spectral risk. The results are shown in Figure~\ref{fig:risk lambda} and Figure~\ref{fig:risk t}.

We observe from Figure~\ref{fig:risk lambda} that: (a) As $\lambda$ decreases, the fitting error steadily decrease. This is as expected. What deserves notice is that the error of the five data sets coincidentally converge at the point $\lambda=10^{-8}$, and reach a plateau of around 1 thereafter. This is consistent with the test error in Figure~\ref{fig:para:lambda}. The convergences confirm the law of $\lambda$: as a rule of thumb, setting $\lambda<10^{-8}$ generally delivers near-optimal result of nRBFN. (b) For most data sets, after $\lambda<10^{-12}$, there exists a fairly stable converging range of spectral risk: $10^{8\sim10}$. This also holds for most other data sets not shown, especially large data sets. The cases of the two gene data are abnormal, they inherently have low spectral risk, variations can be observed only when $\lambda$ becomes large. (c), (d) The weighted spectral risk ($\lambda\gamma$) and its proportion to the tradeoff converge after some points of $\lambda$. They are not comparable to the fitting error.

Next, Figure~\ref{fig:risk t} shows that: (a) As $t$ increases, the fitting error steadily decrease, as expected.
At $t=1$, all error converge close to 1, due to full ranks of the similarity matrices. The fitting capacity of nRBFN is a matter of choice. At $t=0.9$, the error have reached a suitable level, so setting $t=0.9$ as the default value meets general situation. (b) We observe that the spectral risk again converges to the range of $10^{8\sim10}$. The convergences of big data sets are smoother than those of smaller ones. (c), (d) The weighted spectral risk is again not comparable to the fitting error.

It should be noted that the definitions of relative measures are essential, the above laws would disappear, if we simply use the absolute measures.

\begin{figure*}[htb]
\center{
\subfigure[]{\label{fig:risk lambda:eps}
\includegraphics[width=3.6cm]{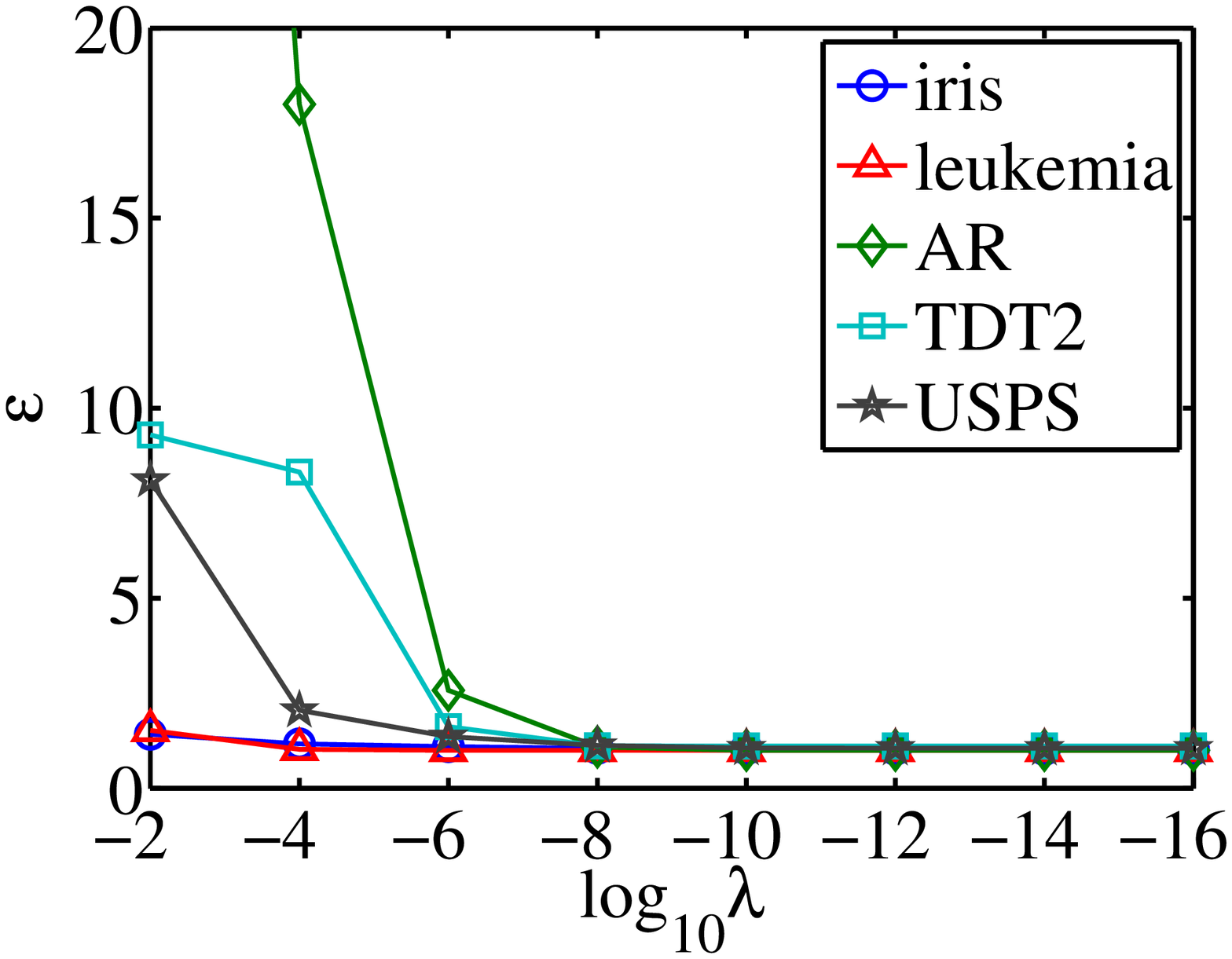}}
\subfigure[]{\label{fig:risk lambda:risk}
\includegraphics[width=3.6cm]{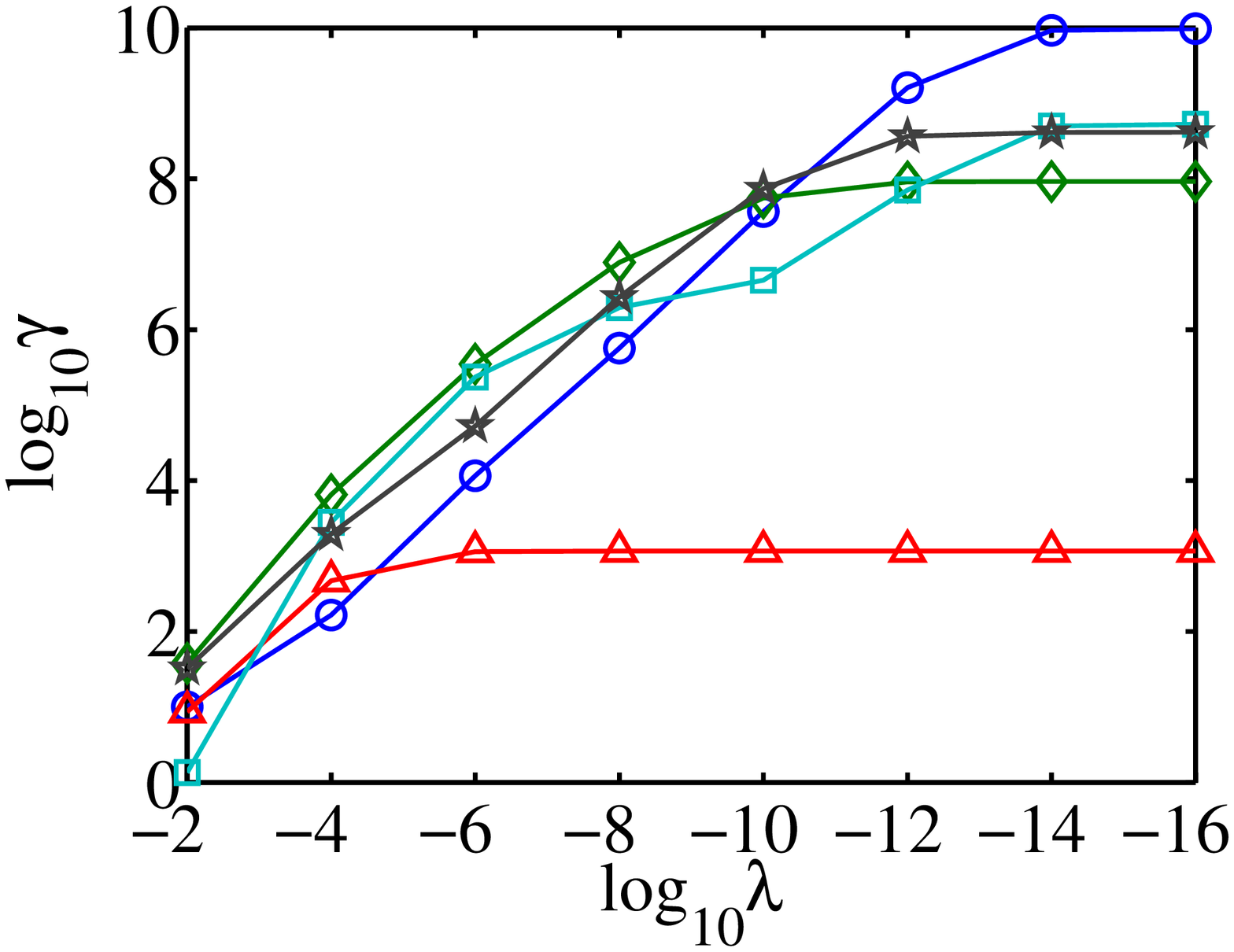}}
\subfigure[]{\label{fig:risk lambda:lamrisk}
\includegraphics[width=3.6cm]{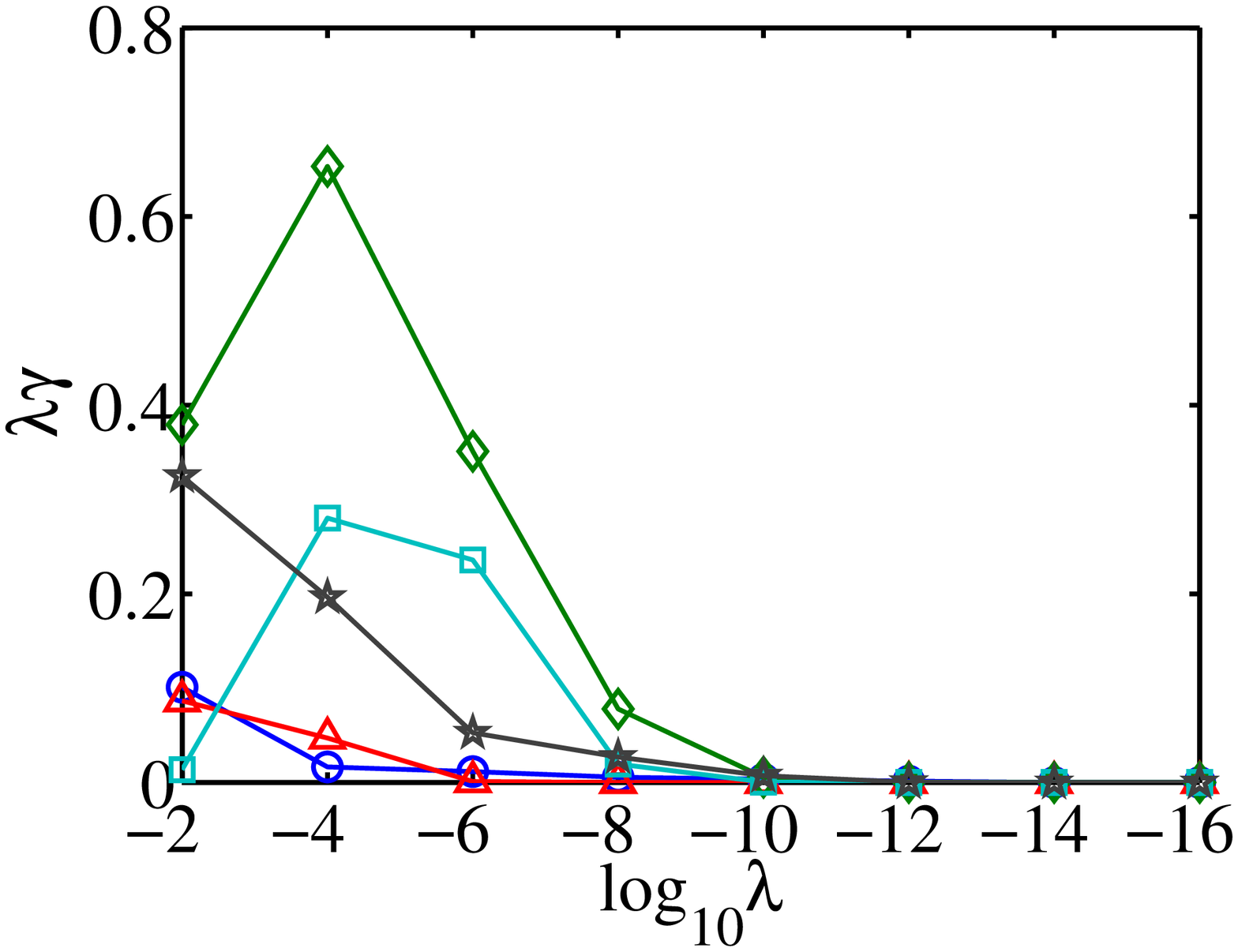}}
\subfigure[]{\label{fig:risk lambda:tra}
\includegraphics[width=3.6cm]{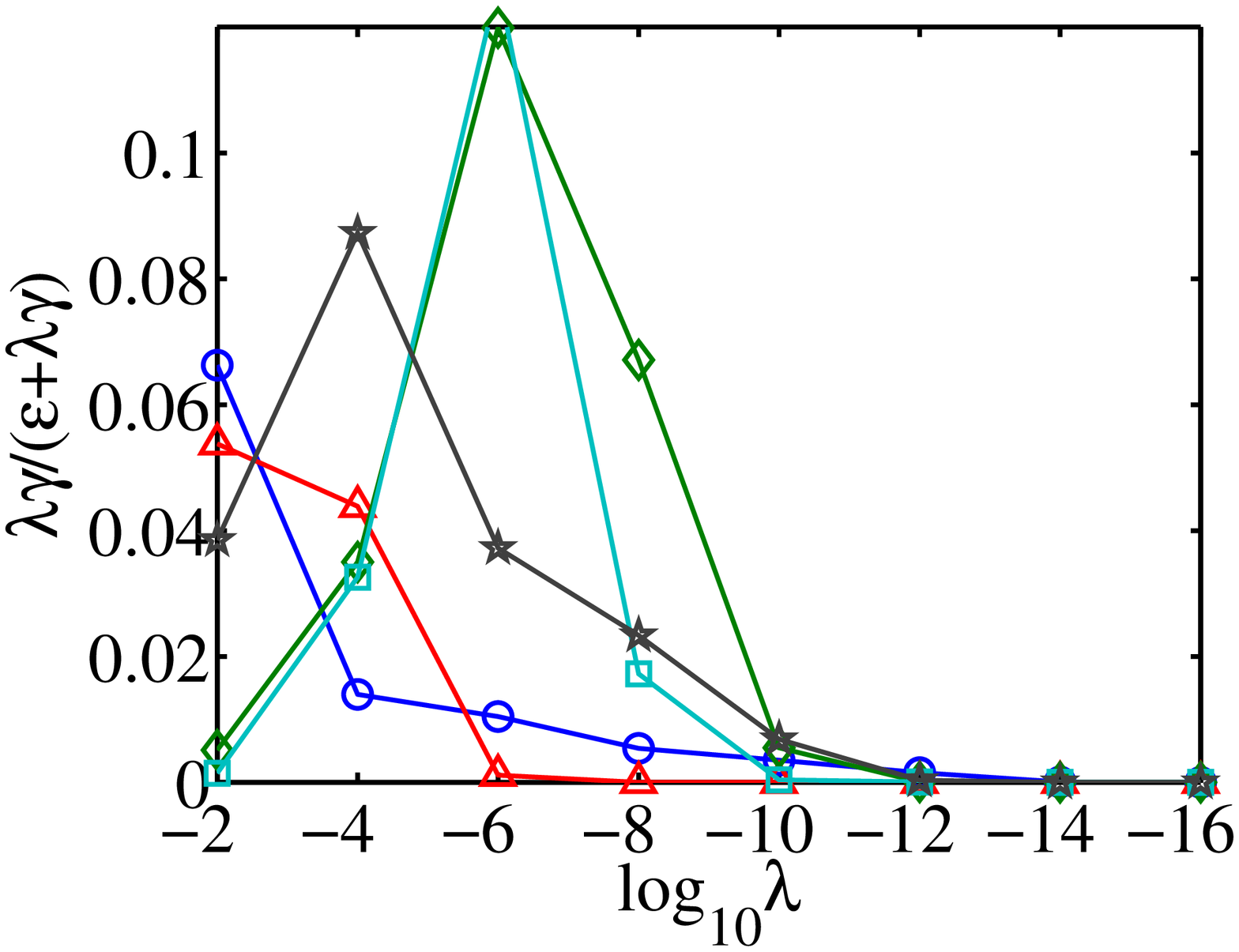}} }
\caption{Changes of the error and risk as $\lambda$ varies.}\label{fig:risk lambda}
\end{figure*}

\begin{figure*}[htb]
\center{
\subfigure[]{\label{fig:risk t:eps}
\includegraphics[width=3.6cm]{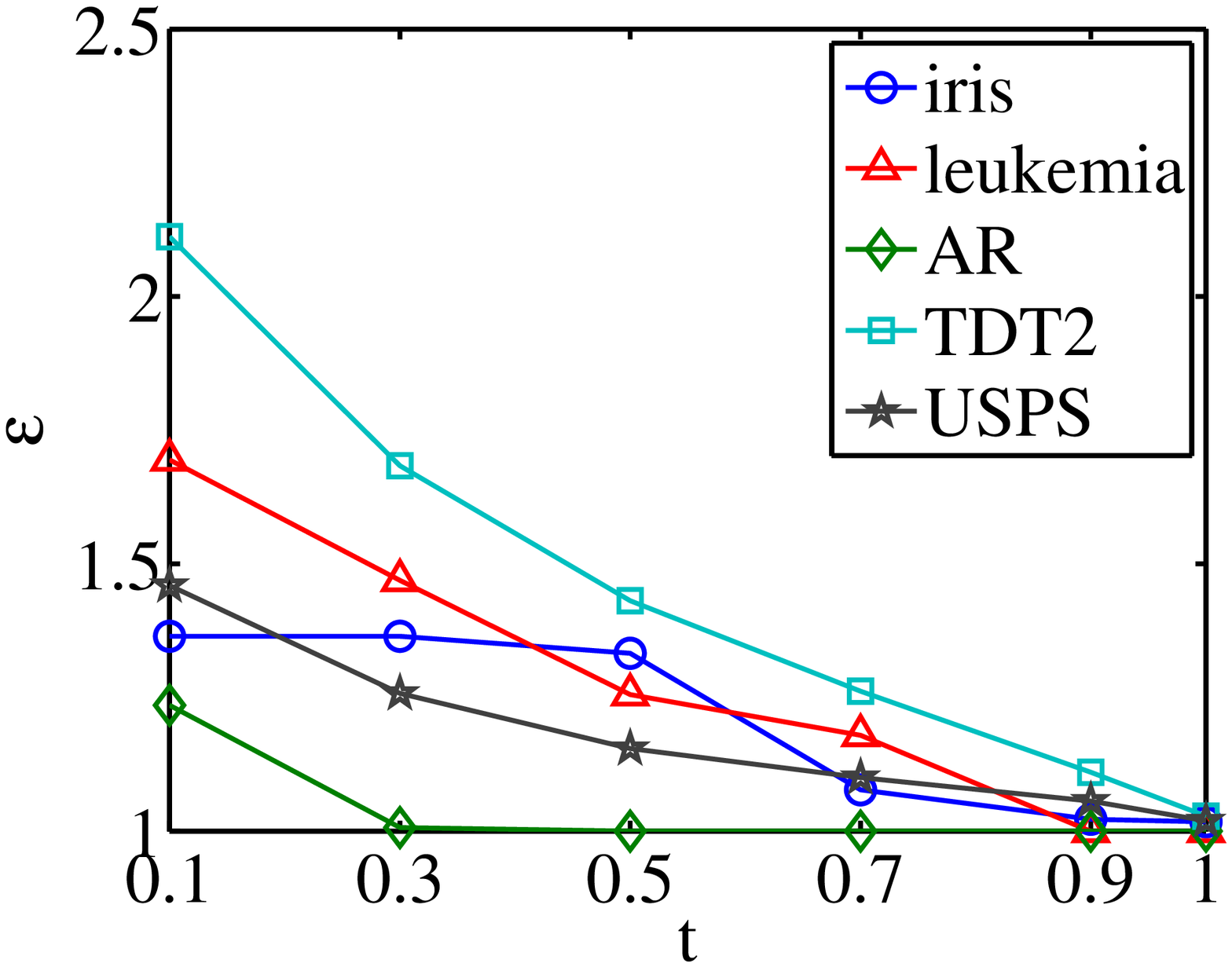}}
\subfigure[]{\label{fig:risk t:risk}
\includegraphics[width=3.6cm]{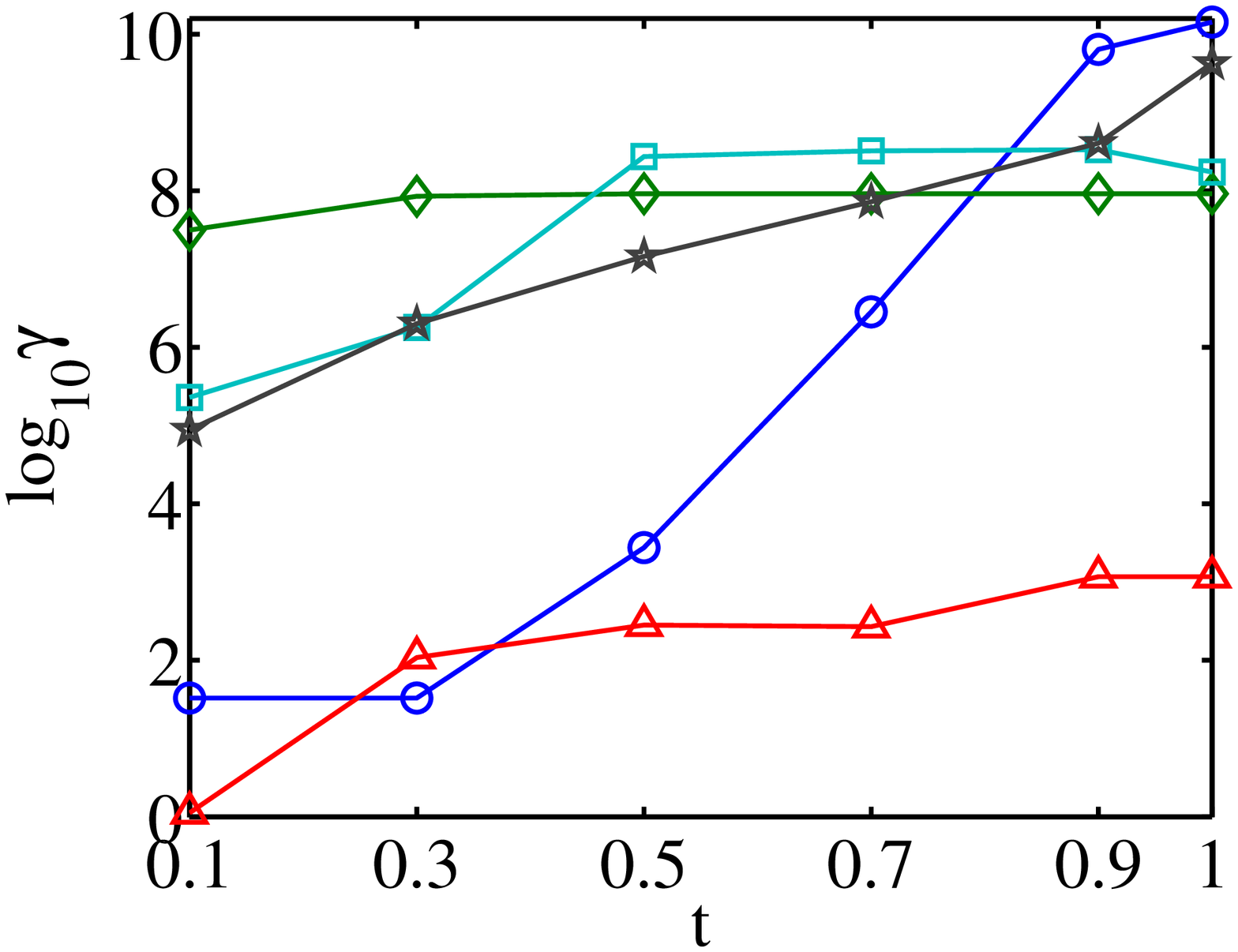}}
\subfigure[]{\label{fig:risk t:lamrisk}
\includegraphics[width=3.6cm]{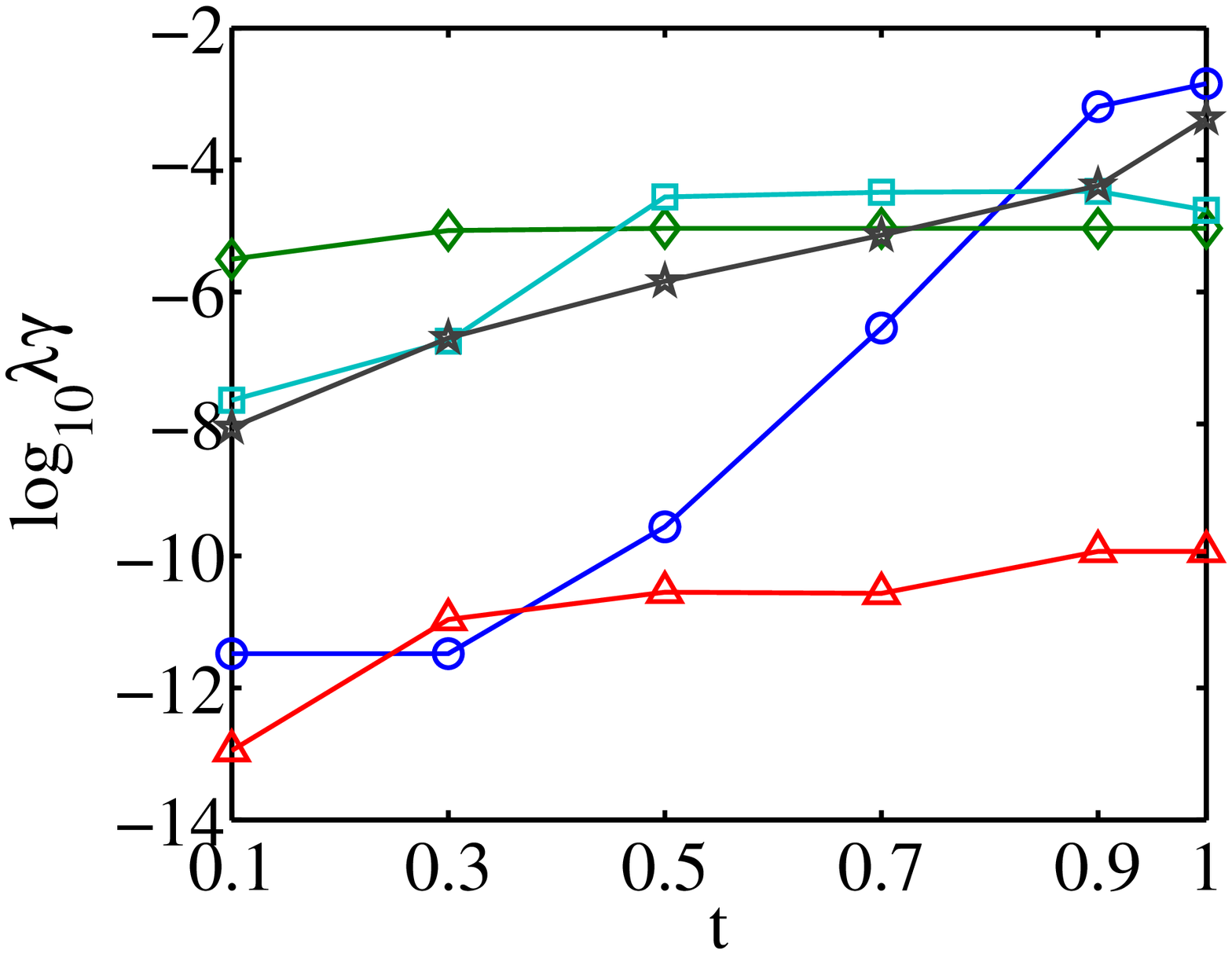}}
\subfigure[]{\label{fig:risk t:tra}
\includegraphics[width=3.6cm]{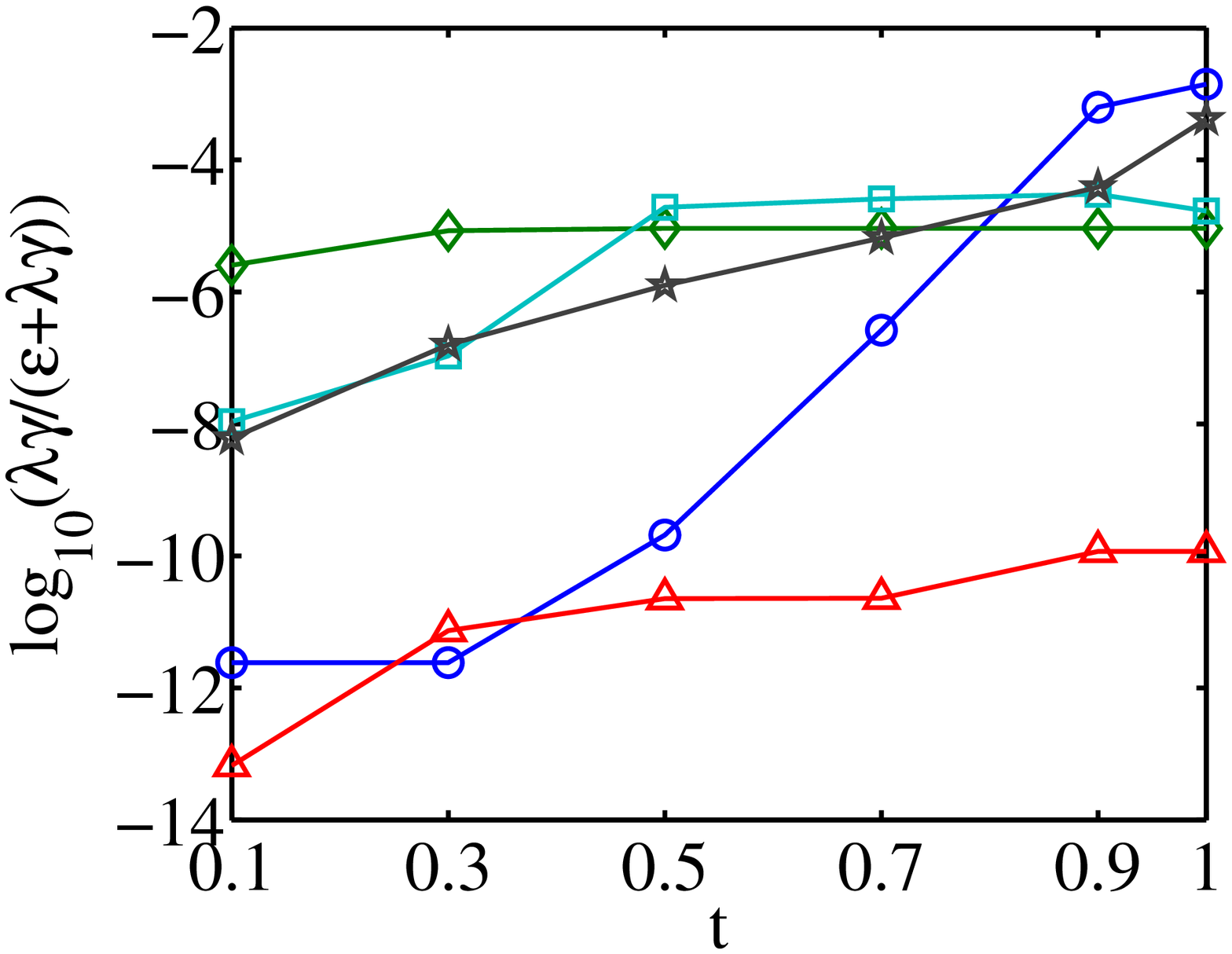}} }
\caption{Changes of the error and risk as $t$ varies.}\label{fig:risk t}
\end{figure*}

\section{Related Work}\label{sec:related work}
We discuss the related work of nRBFN, regularization, and generalization error.

1. nRBFN. nRBFN is initially mentioned by \cite{Moody1989Fast} and later derived by \cite{Specht1990Probabilistic,Specht1991A,Xu1994On} from probability density estimation and kernel regression. It is also closely related to Gaussian mixture model \cite{Tresp1993Network}. Comparing with RBFN, nRBFN has a distinctive feature: in regions far from the samples, the output of RBFN will vanish due to the localized property of radial basis function, while that of nRBFN will not due to the normalization \cite{Bugmann1998Normalized}. Consequently, nRBFN provides better smoothness. Meanwhile, the universal approximation capacity is preserved \cite{Xu1994On,Benaim1994On}. However, the connection of nRBFN to spectral graph theory is absent before.

2. Regularization. The most classical viewpoint to regularization is from the Bayesian probability, see e.g., \cite{bishop2006pattern}. The regularization term corresponds to a prior distribution of the weights, while the error term corresponds to the conditional distribution of the target. However, in the Bayesian view, the effect of regularization on reducing overfitting is not as obvious as the spectral view. Moreover, how to set the regularization weight is unclear, and the uniform case that regularization does not help is not easily observed. Another viewpoint to regularization is from the Tikhonov regularization theory that is based on functional analysis, see e.g., \cite{Poggio1990Networks}. In this view, the regularization term corresponds to a constraint imposing some smoothness on the approximating function.

3. Generalization error. Even though we are concerned with the generalization problem, this paper is limited to the study of overfitting risk. We did not investigate the problem of generalization error or expected risk.\footnote{In the literature, see, e.g., \cite{Niyogi1996On}, the ``risk'' of expected risk actually means error. A similar concept is empirical risk \cite{Vapnik2000The}, actually it is training/fitting error. It should not be confused with the ``risk'' of overfitting/spectral risk in this paper.} Both of these two concepts relate to the error of the approximating function with respect to the underlying data distribution. In classification application, they indicate the error of a classifier when dealing with new data. Results on this problem had been established both for RBFN \cite{Niyogi1996On,Krzyzak1996Nonparametric,Krzyzak1998Radial,Que2016Back} and nRBFN \cite{Xu1994On,Kegl2000Radial}. A typical result states that with probability greater than $1-\delta$, the generalization error of RBFN is upper bounded by $O(1/r)+O(\sqrt{(pr\log(nr)-\log \delta)/n})$ \cite{Niyogi1996On}.

\section{Future Work}\label{sec:future work}
We mention some limitations of the paper as well as future work worthwhile to do.

Concerning the performance improvements of nRBFN: (1) The relation between the ideal graph condition and the basis selection is not yet investigated. Our motivation to develop the basis selection strategy and parameter setting scheme is to demonstrate the practical performance of nRBFN and provide a baseline algorithm that is easy to use. Searching the optimal basis that is consistent with the theory is an important direction of future work. (2) Optimized or approximated search of nearest neighbors, e.g., \cite{Muja2014Scalable}, can be applied to address the bottleneck of speed improvement. (3) The basis can be further reduced, for there may be many boundary points highly overlapping. (4) Online basis learning can be considered \cite{Platt1991A}, where the basis can be increased, updated, or pruned. It will enable (n)RBFN to handle large scale data.

Concerning the theoretical investigations: (1) The error-and-risk analysis of nRBFN in the perturbation case does not depend on the normality of the columns of similarity matrix. This implies that the analysis can serve as a foundation for the analysis of the other models, e.g., RBFN \cite{Que2016Back}, ELM \cite{Huang2006Extreme}. Empirically, we found that when RBFN uses our basis selection strategy, it performs similarly to nRBFN and frequently even better. The perturbation analysis in this paper is limited to tiny noise. Extending the analysis from perturbation case to normal noise case has great practical significance. (2) How the two factors of basis size and ideal graph condition interact and contribute to the performance of nRBFN deserves further study. Many of the empirical laws observed in the experiments require explanations. (3) It will be interesting to compare the spectral risk with the VC dimension \cite{Vapnik2000The}, the error-and-risk tradeoff with the structural risk minimization \cite{Vapnik2000The}.

\appendix

\section{The Gap between $F$ and the Leading Row-subspace of $\tilde{W}$}\label{sec:gap}
Under the ideal graph condition, we know the row-space of $\tilde{W}$ contains $F$. Now we study how much $F$ deviates from the leading row-subspace and when the gap is closed.

For simplicity, we consider one class, since under the ideal graph condition the blocks of matrices of different classes are independent of each other, the largest singular vectors of $\tilde{W}$ consist of the largest singular vectors of each block. We use lower case symbols $w\in \mathbb{R}^{r_k\times n_k}$, $f$, $x$ to denote the nonzero blocks of the $k$th class. Note, $f$ is a uniform row vector of 1. The deviation can be measured using the idea of the $\ell_2$ operator norm \cite{golub1996matrix}:
\[
\psi\doteq \max_{y\neq \mb{0}} \frac{\|y\tilde{w}\|_2^2}{\|y\|_2^2}-\frac{\|x^*\tilde{w}\|_2^2}{\|x^*\|_2^2}\geq 0,
\]
On the one hand, by property of the $\ell_2$ operator norm, $\max_{y\neq \mb{0}} \|y\tilde{w}\|_2^2/\|y\|_2^2=\|u_1\tilde{w}\|_2^2/\|u_1\|_2^2=\|\sigma_1v_1\|_2^2/\|u_1\|_2^2=\sigma_1^2$, where $u_1$, $v_1$, and $\sigma_1$ denote the largest singular vectors and value of $\tilde{w}$. On the other hand, by Proposition~\ref{pro:X indicator}, $x^*$ is a uniform row vector of 1, and $x^*\tilde{w}=f$. Thus,
\[
\psi=\sigma_1^2-n_k/r_k.
\]
$\psi=0$ if and only if $\sigma_1^2=n_k/r_k$, i.e., $x^*$ and $f$ become the largest singular vectors. The remaining effort focuses on the estimation of $\sigma_1$. We have the following proposition:
\begin{proposition}
For any nonnegative matrix $\tilde{w}\in \mathbb{R}^{r_k\times n_k}$ $(r_k\leq n_k)$ with each column sum normalized to 1, assume $\tilde{w}$ is of full rank, then
\begin{enumerate}
  \item $n_k/r_k\leq \sigma_1^2 \leq z_{max}$, where $z_{max}$ is the maximal row sum of $\tilde{w}$.
  \item $n_k/r_k=\sigma_1^2$ if and only if the row sums are even, i.e., $z_{max}=n_k/r_k$.
\end{enumerate}
\end{proposition}

\begin{proof}
1) The problem is resolved with the help of a closely related matrix, $\hat{W}\doteq Z^{-\frac{1}{2}}\tilde{W}$, where $Z$ is a diagonal matrix of the row sums of $\tilde{W}$. $\hat{W}$ is the component of a reduced Laplacian matrix, $\hat{L}\doteq I-\hat{W}^T\hat{W}$, that is also for dealing with the scalable problem of large graph construction \cite{Liu2010Large}. Under the ideal graph condition, $\tilde{W}$ is block-wise, so is the reduced similarity matrix $\hat{W}^T\hat{W}$, which implies $\hat{W}^T\hat{W}$ is ideal too. Therefore, $F$ is the smallest eigenvectors of $\hat{L}$ (eigenvalue 0), and equivalently, the largest right singular vectors of $Z^{-\frac{1}{2}}\tilde{W}$ (singular value 1). Denote $\hat{w}$ and $z$ to be the corresponding blocks of class $k$. We have
\[
\sigma_1^2=\sup_{y\neq \mb{0}} \frac{\|y\tilde{w}\|_2^2}{\|y\|_2^2}=\sup_{y\neq \mb{0}} \frac{\|yz^{\frac{1}{2}}\hat{w}\|_2^2}{\|y\|_2^2}.
\]
Denote $y'=yz^{\frac{1}{2}}$, it turns into
\[
\sigma_1^2=\sup_{y'\neq \mb{0}} \frac{\|y'\hat{w}\|_2^2}{\|y'z^{-\frac{1}{2}}\|_2^2}\leq \sup_{y'\neq \mb{0}} \frac{\|y'\hat{w}\|_2^2}{z_{max}^{-1}\|y'\|_2^2}\leq z_{max}.
\]
$z_{max}\doteq \max_{i} z_{ii}$. In the last inequality, we have used the fact that the largest singular value of $\hat{w}$ is 1.

2) First, note that $n_k/r_k$ is also the mean of the row sums of $\tilde{w}$, since the sum of row sums is equal to the sum of column sums, which is $n_k$. Thus, if the row sums are even, then $z_{max}=n_k/r_k$. Consequently $n_k/r_k=\sigma_1^2$. Conversely, if $n_k/r_k=\sigma_1^2$, then it means $x^*$ and $f$ are the largest singular vectors, since $\|x^*\tilde{w}\|_2^2/\|x^*\|_2^2=n_k/r_k$. In this case, again by the property of operator norm,
\[
\sigma_1^2=\frac{\|\tilde{w}f\|_2^2}{\|f\|_2^2}=\frac{\sum_i z_{ii}^2}{n_k}\geq \frac{(\sum_i z_{ii})^2/r_k}{n_k}=\frac{n_k}{r_k}.
\]
The equality holds, if and only if $z_{ii}$'s are even.
\end{proof}

With this proposition, we immediately have
\begin{proposition}
\[
\psi \leq z_{max}-n_k/r_k.
\]
$\psi=0$ if and only if the row sums are even, i.e., $z_{max}=n_k/r_k$.
\end{proposition}

By this proposition, we finally arrive at Theorem~\ref{theo:F leading}.

\bibliographystyle{plain}
\bibliography{sparse}

\end{document}